%% file: main.tex
\documentclass{article}

\usepackage[main, final]{neurips_2026}

\usepackage[utf8]{inputenc} 
\usepackage[T1]{fontenc}    
\usepackage{hyperref}       
\usepackage{url}            
\usepackage{booktabs}       
\usepackage{amsfonts, amsthm}       
\usepackage{nicefrac}       
\usepackage{microtype}    
\usepackage{multirow}
\usepackage{xcolor}         
\usepackage{titletoc}       
\usepackage{graphicx}
\usepackage{amssymb}
\usepackage{algorithm}
\usepackage{algpseudocode}
\usepackage{mathtools}
\usepackage{wrapfig}
\usepackage{caption}
\usepackage{subcaption}
\usepackage{enumitem}

\usepackage{thmtools, mathtools}

\newtheorem{definition}{Definition}

\newcommand{\given}{\,|\,}
\newcommand{\Chi}{\mathrm{X}}
\theoremstyle{definition}
\newtheorem{example}{Example}

\usepackage{xcolor}

\title{MoRe: Modular Representations for Continual Representation Learning on Sequential Data}

\author{%
    Jiaqi Sun$^{1,2}$ \quad
    Boyang Sun$^{2}$ \quad
    Rasmy M. H.$^{2}$ \quad
    Xiangchen Song$^{1}$ \quad
    Kun Zhang$^{1,2}$ \\
    $^1$Carnegie Mellon University \\
    $^2$Mohamed bin Zayed University of Artificial Intelligence\\
    \texttt{\{jiaqisun, kunz1\}@cmu.edu} \\
}

\begin{document}

\maketitle

\input{cr-tex/0_abs}
\input{cr-tex/1_intro}
\input{cr-tex/3_problem_setting}

\input{cr-tex/4_theory}

\input{cr-tex/5_new_method}
\input{cr-tex/6_new_exp_v2}
\input{cr-tex/2_related_work}
\input{cr-tex/7_conclusions}

\newpage
\bibliographystyle{abbrv}
\bibliography{main}

\newpage
\appendix
\input{cr-tex/8_appendix}


\end{document}

%% file: cr-tex/0_abs.tex
\begin{abstract}
Continual learning requires models to adapt to new data while preserving previously acquired knowledge. At its core, this challenge can be viewed as principled one-step adaptation: incorporating new information with minimal interference to existing representations.
Most existing approaches address this challenge by modifying model parameters or architectures in a supervised, task-specific manner.
However, the underlying issue is representational: tasks require distinct yet structured representations that can be selectively updated without disrupting existing representations, while this structure should reflect intrinsic organization in the data rather than task boundaries.
In sequential data, time-delayed dependencies provide a natural signal for uncovering this structural organization, revealing how fundamental representations give rise to more specific ones.
Inspired by the modular organization of the human brain, we propose MoRe, a framework that derives a specific type of modular representation from the given sequential data and leverage it for continual learning with selective updates. MoRe decomposes knowledge into a hierarchy of fundamental and specific modules with identifiability guarantees, enabling principled module reuse and expansion that preserve old modules by construction, with selective adaptation under limited capacity.
Experiments on synthetic benchmarks and real-world LLM activations demonstrate interpretable hierarchical structure and improved plasticity--stability trade-offs, suggesting that MoRe may serve as a principled foundation for continual learning.
\end{abstract}

%% file: cr-tex/1_intro.tex
\section{Introduction}
With the growing attention to foundation models, training such models in a sequential manner, which is driven by limited learning capacity and the inherently sequential flow of data, has brought renewed attention to the long-standing problem of Continual Learning~(CL), or lifelong learning.
Generally speaking, the research problem of CL is how to avoid  catastrophic forgetting while acquiring new knowledge or abilities during the learning procedure. In essence, without assuming that some part of the model remains stable, avoiding catastrophic forgetting is impossible.

Existing CL methods can be divided into supervised and unsupervised CL. Supervised CL directly models the {abilities} as tasks and aims to enable the model to perform a growing set of tasks as new ones arrive. In this category, representative methods 
regularize parameter or output changes~\cite{kirkpatrick2017overcoming,aljundi2018memory,li2016learning}, replay information from previous tasks~\cite{rebuffi2017icarl}, or isolate parameters across tasks~\cite{rusu2016progressive,fernando2017pathnet,aljundi2017expert,serra2018overcoming,mallya2018packnet}; the architectural variants of the last case recognize that different tasks require different parts of the model~\cite{chen2023lifelong,yu2024boosting,liang2024inflora,wang2023orthogonal,wang2022learning,wang2022dualprompt,smith2023coda}, but how those parts are related is rarely addressed. 
Unsupervised CL, on the other hand, focuses on proper representation learning~\cite{rao2019continual,hu2021well}, yet most of its methods still operate via replay, distillation, regularization, or constraint modification~\cite{madaan2021representational,fini2022self,gomez2022continually,tang2024kaizen,liu2025branch}. Across both categories, decisions about what to preserve and what to update are largely heuristic. 

Accordingly, the core of CL is representational: the specialization of model components determines what can be reused and what ought to be updated. We therefore tackle the problem of \emph{continual representation learning}: how can a model form representations that grow cumulatively, across sequentially observed scenarios, without catastrophic forgetting? This mirrors how cognition operates: understanding is updated continuously throughout life, and tasks are addressed by selecting the relevant pieces of understanding rather than by reshaping the system into a task-specific tool. Once rich representations underlying this understanding are achieved, only a few samples are needed to assemble them for downstream tasks; the substantive challenge is the representation itself.

Inspired by the specialization and composition of modules in the human brain~\cite{tafazoli2026building, tian2026domain}, we propose Modular Representations (MoRe) for principled continual representation learning. The core idea is that representations should consist of modules organized in a hierarchy from \emph{fundamental} to \emph{specific}, where fundamental modules are reused as building blocks for more specific ones. Crucially, this hierarchy should reflect intrinsic organization in the data rather than be imposed by task boundaries or architectural design.

Interestingly, we find that sequential data offer a particularly clear setting for identifying such modular structure. Cross-layer time-delayed effects flow unidirectionally from fundamental to specific layers, providing the asymmetry needed to make the hierarchy {identifiable} up to within-layer indeterminacies and a compatible layer order. Identifiability is what turns modules from a useful inductive bias into well-defined objects: with identifiable modules, decisions about what to reuse and what to expand at each adaptation step become principled rather than heuristic, and the global representation can be recovered incrementally from a sequence of partial views.

Our contributions are as follows.
(1) We propose Modular Representations (MoRe) for sequential data, organizing latent mechanisms into a hierarchy grounded in the data-generating process and supporting selective, minimal updates across stages.
(2) We establish identifiability guarantees under cross-layer time-delayed dependencies and develop estimation and reuse procedures, with preservation-first expansion and selective adaptation under capacity constraints.
(3) On synthetic benchmarks, we demonstrate factor recovery, accurate reuse, compact library growth, and plasticity–stability trade-offs. On real-world LLM activation data, we evaluate representation retention, category selectivity, and downstream informativeness in one-step and continual learning settings.

%% file: cr-tex/3_problem_setting.tex
\section{Problem Setting}

\paragraph{General Continual Representation Learning}

\begin{wrapfigure}{r}{0.5\linewidth}
    \centering
    \vspace{-0.5cm}
    \includegraphics[width=\linewidth, trim={0cm, 4cm, 0cm, 5cm}]{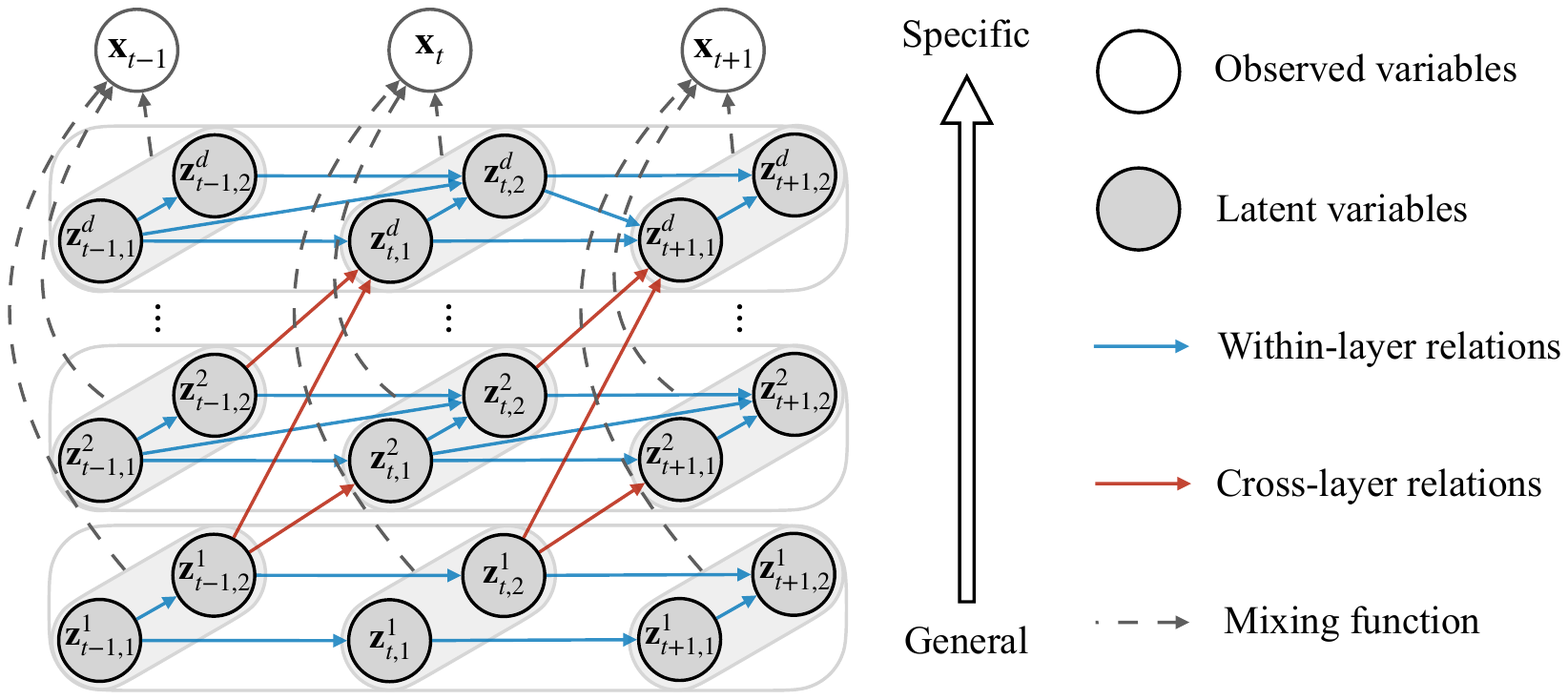}
    \caption{A hierarchical generative model for sequential data. Cross-layer influences flow from general to specific layers at positive time lags; within-layer effects may be instantaneous or time-delayed.}
    \vspace{-0.5cm}
    \label{fig:hierarchical_model}
\end{wrapfigure}

We consider observations $\mathbf{X}_t \in \mathcal{X}$ from an underlying global time-series process, where $t$ indexes time. At learning stage $s$, the learner receives a batch of data $\mathcal{D}_s$ from a stage-specific distribution $P_s$, with $\operatorname{supp}(P_s) \subseteq \operatorname{supp}(P)$, where $P$ denotes the corresponding global observation distribution. Observations within each batch retain their temporal dependence. Different stages may expose different regions of the latent space or different combinations of latent factors, while sharing the generative mechanisms defined below. 
Each batch therefore provides only a partial view of the global process.

The goal is to learn an encoder $f : \mathcal{X} \rightarrow \tilde{\mathcal{Z}}$ incrementally from $\mathcal{D}_0,\mathcal{D}_1,\ldots,\mathcal{D}_s$, without requiring joint access to all batches. 
The encoder should use newly observed data to progressively capture more of the latent structure while retaining its ability to represent previously observed data.

We refer to this setting as \emph{General Continual Representation Learning}. To support selective adaptation, we seek a representation structure that specifies which components can be reused and which require updating or expansion. We distinguish the underlying generative variables $\mathbf{Z}_t$ from the learned representations $\tilde{\mathbf{Z}}_t=f(\mathbf{X}_t)$.

\paragraph{Data Generating Process with Modular Representations}

We view real-world observations as arising from structured processes where general knowledge influences more specific knowledge, as characterized later. This motivates organizing the latent representations into a hierarchy: lower layers capture general, reusable factors, and higher layers capture more specific factors whose computation depends on them.

Specifically, the observation $\mathbf{X}_t$ is generated from $d$ latent layers $\mathbf{Z}_t^1,\ldots,\mathbf{Z}_t^d$, with $\mathbf{Z}_t^i\in\mathbb{R}^m$, via
\begin{align}
    \mathbf{X}_t &= g(\mathbf{Z}_t) \in \mathcal{X}\subseteq\mathbb{R}^n, &
    \mathbf{Z}_t &=
    \big[(\mathbf{Z}_t^1)^\top,\ldots,
         (\mathbf{Z}_t^d)^\top\big]^\top
    \in \mathcal{Z}\subseteq\mathbb{R}^{dm}.
    \label{eq: hierarchical_model_mixing_func}
\end{align}
Here, $g:\mathcal{Z}\rightarrow\mathcal{X}$ is a nonlinear mixing function that is invertible. We assume $dm\leq n$ and the regularity conditions specified in Section~\ref{sec: theory}. Among the layers $\mathbf{Z}_t^i$, $i = 1,\cdots,d$, we distinguish fundamental ones from specific ones in the sense that the former can generate the latter, with a time delay, but not vice versa.   Layers are indexed from general to specific, i.e., if $i<j$, $\mathbf{Z}_t^i$ is more fundamental than $\mathbf{Z}_t^j$. 


Within a layer, the latents operate at
a comparable level of abstraction, so we allow both instantaneous and time-delayed interactions.
Thus, for the $k$-th dimension of layer $i$, we have
\begin{align}
    Z_{t,k}^{i}
    = h_{i,k}\Big(
        \mathbf{Z}_{t,\mathrm{pa}_i(k)}^{i},
        \{\mathbf{Z}_{t-\tau}^{i}\}_{\tau=1}^{L},
        \{\mathbf{Z}_{t-\tau}^{j}:
            1\leq j<i,\ 1\leq\tau\leq L\},
        \epsilon_{t,k}^{i}
      \Big),
    \label{eq: hierarchical_model_layer_z}
\end{align}
where $\mathrm{pa}_i(k)$ denotes the instantaneous parents of coordinate $k$ within layer $i$, and $L$ is the maximum lag.
We assume that instantaneous dependencies are acyclic and that the exogenous noise variables are mutually independent across coordinates, layers, and time.
Note that the index $t$ denotes timesteps within the sequential data, whereas later we use stages to refer to the steps of continual learning.

\begin{example}[A Simple Mathematical World]
Consider a synthetic mathematical world with two modules: $\mathbf{Z}_t^1$ represents algebraic knowledge, and $\mathbf{Z}_t^2$ represents knowledge of integration. Solving an integration problem may require an algebraic transformation, such as expanding $(x+1)^2$ before integrating it. This dependence is modeled as a time-delayed influence from the algebra module to the integration module, with no influence in the reverse direction. 
Within each module, instantaneous dependencies can relate mathematical expressions at the same timestep. For instance, in the algebra module, the terms $x^2$, $2x$, and $1$ determine the polynomial $x^2+2x+1$. We model these instantaneous dependencies as acyclic. Time-delayed dependencies within a module may operate in either direction: intermediate results can support later computations, while a combined expression can guide its subsequent decomposition.
\end{example}


\paragraph{From Modular Generative Process to Continual Learning}



This generative structure motivates an encoder composed of
layer-wise modules. We will show in Section~\ref{sec: theory} that, under proper assumptions, the latent layers are identifiable up to the within-layer ambiguities and admissible layer orderings specified there. This result provides a correspondence between learned modules and underlying latent layers, giving a structural basis for selective adaptation.

%% file: cr-tex/4_theory.tex
\section{Theoretical Guarantees}
\label{sec: theory}

We establish the identification properties of the proposed hierarchical generative model in two steps. 
First, Theorem~\ref{thm1: latent variables} excludes cross-layer mixing while permitting transformations within a layer.
Second, Theorem~\ref{thm2: latent hierarchy} shows that the hierarchical mask produces a compatible organization and gives a sufficient condition for exact recovery.
For simplicity, we consider the single-lag case, $L=1$;  this can be immediately extended to multiple lags, as shown in \cite{li2025idol} for a related problem. Formal assumptions and proofs are provided in the appendix.

Throughout, the true and estimated models satisfy the generative restrictions introduced above, with $d$ layers and $m$ latents per layer known during estimation.
Observational equivalence of two models means that they induce the same distribution over observation sequences.

\begin{restatable}[Identifiability of Latent Variables]
{theorem}{latentvariableidentifiability}
\label{thm1: latent variables}

Suppose the observed sequential data $\mathbf{X}_t$ follow the data-generating process with modular representations, and that the true and estimated models
are observationally equivalent.
We state the conditions informally; their mathematical treatment is given in Appendix~\ref{app: proof_thm1}:

\begin{itemize}
    \item[A1] (Smooth and Positive Density)
    The latent transition density varies smoothly and remains positive throughout the latent domain.

    \item[A2] (Sufficient Variability)
    Different past latent values produce sufficiently distinct changes in the current conditional distribution to rule out cross-layer mixing.

    \item[A3] (Sparse Latent Process for Each Layer)
    Within-layer dependency patterns distinguish individual latents from their neighbors, either directly or through a child.

    \item[A2-ind] (Reduced Sufficient Variability)
    Without instantaneous dependencies, variation with the past need only distinguish individual latents.
\end{itemize}

Under A1--A2, any recovered latent is a mixture of only the latents within the same layer in the true model. 
Further assuming A3 implies component-wise identifiability, i.e., each recovered latent is a function of only one true latent. Additionally, in the special case without instantaneous dependencies, A1 and A2-ind suffice for component-wise identifiability, which excludes cross-layer mixing, even though A2 may not hold.
\end{restatable}

\paragraph{Proof Sketch}
Since instantaneous dependencies occur only within a layer, conditional on the past, current latent layers are mutually independent. Combining this factorization with sufficient variability rules out cross-layer mixing, leaving only within-layer transformations. Under the additional conditions for the A3 refinement, existing identification results resolve this remaining ambiguity. Without instantaneous dependencies, the argument reduces to the conditionally independent setting of \cite{yao2022tdrl}. The full proof is provided in Appendix~\ref{app: proof_thm1}.

Having constrained latent mixing, we next characterize the hierarchical organization of the recovered latents. 
An assignment of their order is compatible if no true time-delayed influences go from a later estimated layer to an earlier one, as shown below.

\begin{restatable}[Identification of a Compatible Hierarchy]
{theorem}{hierarchyidentifiability}
\label{thm2: latent hierarchy}
Suppose the conditions of Theorem~\ref{thm1: latent variables} hold, or there are no instantaneous relations. Assume:
\begin{itemize}
    \item[A4] (Hierarchical Mask) Each estimated layer contains $m$ latents. Cross-layer delayed effects respect the estimated order, while instantaneous effects remain within layers.
\end{itemize}
Then any observationally equivalent estimated model
produces a compatible hierarchical assignment.

Further assume the following condition holds.
\begin{itemize}
    \item[A5] (Complete Cross-layer Dependencies) Every latent in an earlier true layer has a time-delayed effect on every latent in each later true layer.
\end{itemize}
Then the ordered layer partition is uniquely identified.
The formal conditions are given in Appendix~\ref{app: proof_thm2}.
\end{restatable}

\paragraph{Proof Sketch}
The identification result preserves dependencies between recovered blocks, or between individual latents in the conditionally independent case. An assignment that reverses a required dependency would exclude it through the hierarchical mask,
contradicting observational equivalence. Therefore, the estimated assignment must be compatible.
Under A5, an incorrect first layer necessarily displaces a true first-layer latent and reverses a required dependency. Applying the same argument successively identifies all layers.
The full proof is provided in Appendix~\ref{app: proof_thm2}.

These results provide the structural basis for the estimation and adaptation procedures developed in Section~\ref{sec: method}, which determine how modules are learned, reused, refined, or expanded.

%% file: cr-tex/5_new_method.tex
\section{Method}
\label{sec: method}
The identifiability results provide the theoretical basis for learning modules from sequential data. Building on them, we propose MoRe, a framework for continual representation learning on sequential data.
MoRe organizes knowledge as a library of hierarchical latent modules, storing for each module its temporal mechanism and the decoder mapping that connects it to the observations.
At each stage, MoRe estimates the structure required by the incoming data, identifies the saved modules that can be reused, and learns the remaining structure. Figure~\ref{fig:more_overview} illustrates the procedure on a selective-reuse example. We consider two operating regimes: \textit{preservation-first learning}, which leaves the library unchanged and expands it when necessary, and \textit{capacity-constrained learning}, which also permits selective adaptation within explicit storage limits.
The estimation objective is stated in general form. From the construction of the library onward, the main text presents each step for linear mixing with $dm<n$, and the nonlinear counterparts are given in Appendix~\ref{app: nonlinear method}.

\begin{figure}[t]
    \centering
    \includegraphics[width=0.93\linewidth]{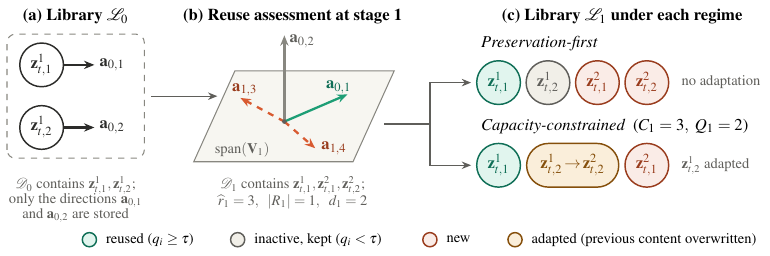}
    \caption{Illustration of MoRe on a selective-reuse example (scenario C4 of Section~\ref{sec:experiments}). (a) The stage-0 library stores decoder directions, but no data. (b) At stage~1, the data $\mathcal{D}_1$ have estimated rank $\widehat r_1=3$. The direction $\mathbf{a}_{0,1}$ lies in the incoming span and is reused, while $\mathbf{a}_{0,2}$ does not, so the reused set is $R_1=\{1\}$ and $d_1=\widehat r_1-|R_1|=2$ new directions are needed. (c) Preservation-first learning keeps $\mathbf{z}^1_{t,2}$ inactive and appends two coordinates; capacity-constrained learning (at most $C_1=3$ stored coordinates and $Q_1=2$ updates) adds one and adapts the least compatible one.}
    \label{fig:more_overview}
\end{figure}

\subsection{Initial Estimation and the Module Library}
\label{sec: initial estimation}

We first estimate the modular representation from the initial dataset $\mathcal{D}_0$, which initializes the library. Guided by Theorems~\ref{thm1: latent variables} and~\ref{thm2: latent hierarchy}, the estimator must explain the observations and enforce the assumed temporal structure, the two requirements under which the latent layers are identifiable.
Let \(\mathrm{Enc}_{\theta}\) map observations to estimated latents \(\widehat{\mathbf{Z}}_t = \mathrm{Enc}_{\theta}(\mathbf{X}_t)\), partitioned into \(d\) layers of \(m\) coordinates, and let \(p_{\psi}\bigl(\widehat{\mathbf{Z}}_t \given \{\widehat{\mathbf{Z}}_{t-\tau}\}_{\tau=1}^{L}\bigr)\) model their temporal dependencies. When the component-wise conditions of Theorem~\ref{thm1: latent variables} hold, each coordinate is identified individually. When \(n = dm\) and \(\mathrm{Enc}_{\theta}\) is invertible, we solve
\vspace{-0.2cm}
\begin{align}
  \min_{\theta,\psi}\;
  \mathbb{E}_t\Bigl[
    -\log\bigl|\det \mathbf{J}_{\mathrm{Enc}_{\theta}}(\mathbf{X}_t)\bigr|
    -\log p_{\psi}\bigl(\widehat{\mathbf{Z}}_t \given \{\widehat{\mathbf{Z}}_{t-\tau}\}_{\tau=1}^{L}\bigr)
  \Bigr]
  + \lambda_{\mathrm{h}}\,\mathcal{R}_{\mathrm{h}}(\theta,\psi),
  \label{eq: initial objective}
\end{align}
where the first two terms form the conditional negative log-likelihood under the change of variables, and the regularizer \(\mathcal{R}_{\mathrm{h}}\) promotes sparse dependencies within the permitted hierarchical structure (A4). If \(\mathrm{Enc}_{\theta}\) is a normalizing flow, invertibility holds by construction and~\eqref{eq: initial objective} is exact maximum likelihood. Appendix~\ref{app: initial estimation} gives the construction of \(\mathcal{R}_{\mathrm{h}}\) and the objective for nonlinear mixing with \(dm<n\).

\paragraph{Constructing the module library.}
The fitted model consists of an encoder, a decoder, and a temporal model, and the initial library \(\mathcal{L}_0\) stores all three together with the preprocessing parameters. No observations are stored. The decoder is the component that later stages rely on. Because the mixing function is shared, the way each coordinate maps to the observations stays fixed across stages, whereas the encoder is fitted to the factors present in \(\mathcal{D}_0\) and may no longer be correct once new factors appear. Reuse assessment (Section~\ref{sec:reuse}) therefore compares the saved decoder mappings with the incoming data.

In the linear case with \(dm<n\), we fit a basis \(\mathbf{V}_0\in\mathbb{R}^{n\times dm}\) from the top principal directions of \(\mathcal{D}_0\), with mean \(\boldsymbol{\mu}_0\), and apply~\eqref{eq: initial objective} to the projected data \(\mathbf{V}_0^\top(\mathbf{X}_t-\boldsymbol{\mu}_0)\), for which the encoder is a square matrix. The encoder and decoder are
\begin{align}
    \widehat{\mathbf{Z}}_t
    = \mathrm{Enc}_0(\mathbf{X}_t)
    = \mathbf{F}_0^\top\mathbf{V}_0^\top(\mathbf{X}_t-\boldsymbol{\mu}_0),
    \qquad
    \widehat{\mathbf{X}}_t
    = \mathrm{Dec}_0(\widehat{\mathbf{Z}}_t)
    = \boldsymbol{\mu}_0+\mathbf{V}_0\mathbf{F}_0^{-\top}\widehat{\mathbf{Z}}_t,
    \label{eq: linear enc dec}
\end{align}
where \(\mathbf{F}_0\in\mathbb{R}^{dm\times dm}\) is the invertible transformation learned with~\eqref{eq: initial objective}. We write \(\mathbf{A}_0=\mathbf{V}_0\mathbf{F}_0^{-\top}\) for the decoder matrix and \(\mathbf{a}_{0,i}=(\mathbf{A}_0)_{:,i}\) for its \(i\)-th column, the observation-space direction of coordinate \(i\).

If joint estimation of all layers is computationally infeasible, the model can instead be estimated layer by layer, using the conditional independence relations of the asymmetric temporal process (Appendix~\ref{app: layer-wise estimation}).

\subsection{Reuse Assessment}
\label{sec:reuse}

After the library is initialized from $\mathcal{D}_0$, each later stage $s\geq 1$ begins with two questions: how many latent coordinates do the incoming data $\mathcal{D}_s$ require, and which coordinates saved in the library can be reused? Both are answered from $\mathcal{D}_s$ and the saved library alone.

Let $\boldsymbol{\mu}_s$ be the mean of $\mathcal{D}_s$, and let $\mathbf{V}_s\in\mathbb{R}^{n\times\widehat r_s}$ contain the top $\widehat r_s$ principal directions of the centered observations, where the width $\widehat r_s$ is the number of singular values above a relative tolerance $\varepsilon$ (Appendix~\ref{app: adaptation-details}). The columns of $\mathbf{V}_s$ form an orthonormal basis of the subspace occupied by the incoming data.
Let $\mathbf{A}_{s-1}$ collect the decoder directions of all coordinates in the library after stage $s-1$, and let $\mathbf{a}_{s-1,i}$ denote
its $i$-th column. Because the mixing function is shared across stages, coordinate $i$ contributes to the observations only along $\mathbf{a}_{s-1,i}$; hence, if it is active in $\mathcal{D}_s$, this direction must lie in the incoming span. We therefore define the compatibility score
\begin{equation}
    q_i =
    \frac{\|\mathbf{V}_s^\top\mathbf{a}_{s-1,i}\|_2}
         {\|\mathbf{a}_{s-1,i}\|_2}\in[0,1],
    \label{eq: compatibility}
\end{equation}
the cosine of the angle between $\mathbf{a}_{s-1,i}$ and the incoming span: $q_i=1$ if the direction lies in the span and $q_i=0$ if it is orthogonal to it.
Under component-wise identifiability (Theorem~\ref{thm1: latent variables}), each coordinate is recovered up to permutation and rescaling, and $q_i$ is invariant to both.
Both $\mathbf{a}_{s-1,i}$ and $\mathbf{V}_s$ live in the observation space, which all stages share, whereas PCA and latent coordinates are stage-specific and not directly comparable.

We select for reuse the coordinates with $q_i\geq\tau$ for a threshold $\tau\in(0,1]$, keeping only the $\widehat r_s$ with the highest scores when more than $\widehat r_s$ pass.
Let $R_s$ be the index set of the selected coordinates. The remaining $d_s=\widehat r_s-|R_s|$ dimensions of the incoming span are not explained by the library and are learned as new coordinates (Section~\ref{sec:preservation}).
The nonlinear counterpart, which matches decoder responses instead of directions, is given in Appendix~\ref{app: nonlinear method}.

\subsection{Preservation-First Learning with Expansion}
\label{sec:preservation}

Given the reused coordinates $R_s$ and the number $d_s$ of new coordinates needed, the preservation-first regime learns the incoming structure without modifying anything already in the library. MoRe reuses the coordinates in $R_s$, adds $d_s$ new coordinates, and keeps the remaining saved coordinates in the library, inactive at this stage, for possible later reuse.

Let $\mathbf{A}_{s-1,R_s}$ denote the columns of $\mathbf{A}_{s-1}$ indexed by $R_s$. The decoder at stage $s$ is
\begin{equation}
    \widehat{\mathbf{X}}_t
    =
    \boldsymbol{\mu}_s
    +\mathbf{A}_{s-1,R_s}\widehat{\mathbf{Z}}_{t,R_s}
    +\mathbf{A}_{s,\mathrm{new}}\widehat{\mathbf{Z}}_{t,\mathrm{new}},
    \label{eq: preservation decoder}
\end{equation}
where $\widehat{\mathbf{Z}}_{t,R_s}$ and $\widehat{\mathbf{Z}}_{t,\mathrm{new}}$ are the reused and the $d_s$ new coordinates, and $\mathbf{A}_{s,\mathrm{new}}\in\mathbb{R}^{n\times d_s}$ holds the new decoder directions. To initialize the new coordinates, we fit~\eqref{eq: initial objective} on $\mathcal{D}_s$, discard the fitted coordinates whose directions match reused ones, and keep the rest. The reused directions $\mathbf{A}_{s-1,R_s}$ then stay fixed while $\mathbf{A}_{s,\mathrm{new}}$, the stage-$s$ encoder, and the permitted temporal interactions of the new coordinates are learned.
The library decoder becomes $\mathbf{A}_s=[\mathbf{A}_{s-1},\mathbf{A}_{s,\mathrm{new}}]$, and the library appends this stage's preprocessing, encoder, and new transition terms without overwriting earlier ones.

New directions need not lie in the span of earlier bases, and stage order need not match causal order: a new coordinate may be upstream of an existing one. When both are observed in $\mathcal{D}_s$, such new-to-old dependencies are learned as additional transition terms in the stage-$s$ configuration, while the saved transitions among old coordinates stay frozen. The nonlinear counterpart, in which a shared decoder replaces the direction matrix, is given in Appendix~\ref{app: nonlinear method}.

\subsection{Selective Adaptation under Capacity Constraints}
\label{sec:capacity}

When storage is limited, MoRe may also modify existing coordinates. In this regime the library holds a single current configuration, which is modified in place, rather than one configuration per stage. Let $C_s$ be the maximum number of coordinates the library may hold, $Q_s$ the maximum number of coordinates added or modified at stage $s$, and $|\mathcal{L}_{s-1}|$ the number of coordinates currently in the library. MoRe first adds
\begin{equation}
  d_s^{\mathrm{add}}=\min\{d_s,\;C_s-|\mathcal{L}_{s-1}|,\;Q_s\}
  \label{eq: capacity allocation}
\end{equation}
new coordinates and then adapts at most $Q_s-d_s^{\mathrm{add}}$ existing ones; if $d_s$ exceeds what these limits allow, the remaining incoming structure is left unmodeled.

In the linear case, adapting coordinate $i$ updates its encoder row, its decoder direction $\mathbf{a}_{s-1,i}$, and its temporal equation. Coordinates are ranked for adaptation by
\begin{equation}
  \pi_i=(1-q_i)+\beta\,\Delta_i^{\mathrm{tem}},
  \label{eq: adaptation priority}
\end{equation}
larger first, where $q_i$ is the compatibility score in~\eqref{eq: compatibility}, $\beta\geq 0$ is a weight, and $\Delta_i^{\mathrm{tem}}$ is the gain in temporal prediction obtained by refitting the transition of coordinate $i$ on $\mathcal{D}_s$; the second term is needed because dynamics can change without changing the observation span. Among highly ranked coordinates, MoRe prefers those with fewer descendants in the estimated transition graph, since modifying them affects fewer downstream coordinates.

The selected parameters are trained on $\mathcal{D}_s$ with~\eqref{eq: initial objective}, with all other parameters frozen; the adapted directions replace the corresponding columns of $\mathbf{A}_{s-1}$, and the new ones are appended, giving $\mathbf{A}_s$. In the nonlinear case, coordinate effects share network parameters, so permitted changes are isolated with a frozen backbone and coordinate-associated residual networks (Appendix~\ref{app: nonlinear method}).




%% file: cr-tex/6_new_exp_v2.tex
\section{Experiments}
\label{sec:experiments}

Our experiments\footnote{The code repository is at \url{https://github.com/sajqavril/modular_representation}.} examine whether the hierarchical modular representation is valid, and how continual learning performs when guided by it. Synthetic experiments with known ground truth examine identification, reuse decisions, comparison with continual-learning baselines, robustness to stage order, and adaptation under limited capacity (S-RQ1--5). Real-world LLM activation experiments examine continual learning, label efficiency, and whether coordinates are selective for semantic categories (RW-RQ1--3). Protocols and full tables are in Appendix~\ref{app: experimental details}, and limitations in Appendix~\ref{sec:limitations}.

\subsection{Synthetic Experiments}

\paragraph{Setup.}
We generate sequential data following Equations~\eqref{eq: hierarchical_model_mixing_func} and~\eqref{eq: hierarchical_model_layer_z}, with both linear ($dm\in\{4,\dots,128\}$) and nonlinear ($dm\in\{4,\dots,16\}$) mixing across $d\in\{2,3,4\}$ layers. S-RQ1 sweeps the latent dimensionality and depth. For S-RQ2--4, we construct 16 two-stage scenarios, C1--C9 and C11--C17, that vary which factors are observed at stages~0 and~1 under the same mixing function. With two layers of two factors, C1--C7 start from the first-layer factors and cover no shift, full expansion, replacement of all old factors, selective reuse with expansion (C4, C5), and retention of one factor, while C8 and C9 add an unseen first-layer factor. C11--C17 mirror C1--C7 with three layers of five factors. S-RQ4--5 also use longer streams (Appendix~\ref{app:transitions}).
We measure latent recovery by the Mean Correlation Coefficient (MCC), reuse decisions by gate accuracy, and reconstruction or prediction by the normalized mean squared error (NMSE). Forgetting is measured by the drop in old-factor MCC after adaptation. Ground-truth factors, identities, and ranks are used only for evaluation.

\begin{figure*}[t]
    \centering
    \begin{minipage}[t]{0.59\textwidth}
        \vspace{0pt}
        \centering
        \includegraphics[width=\linewidth]{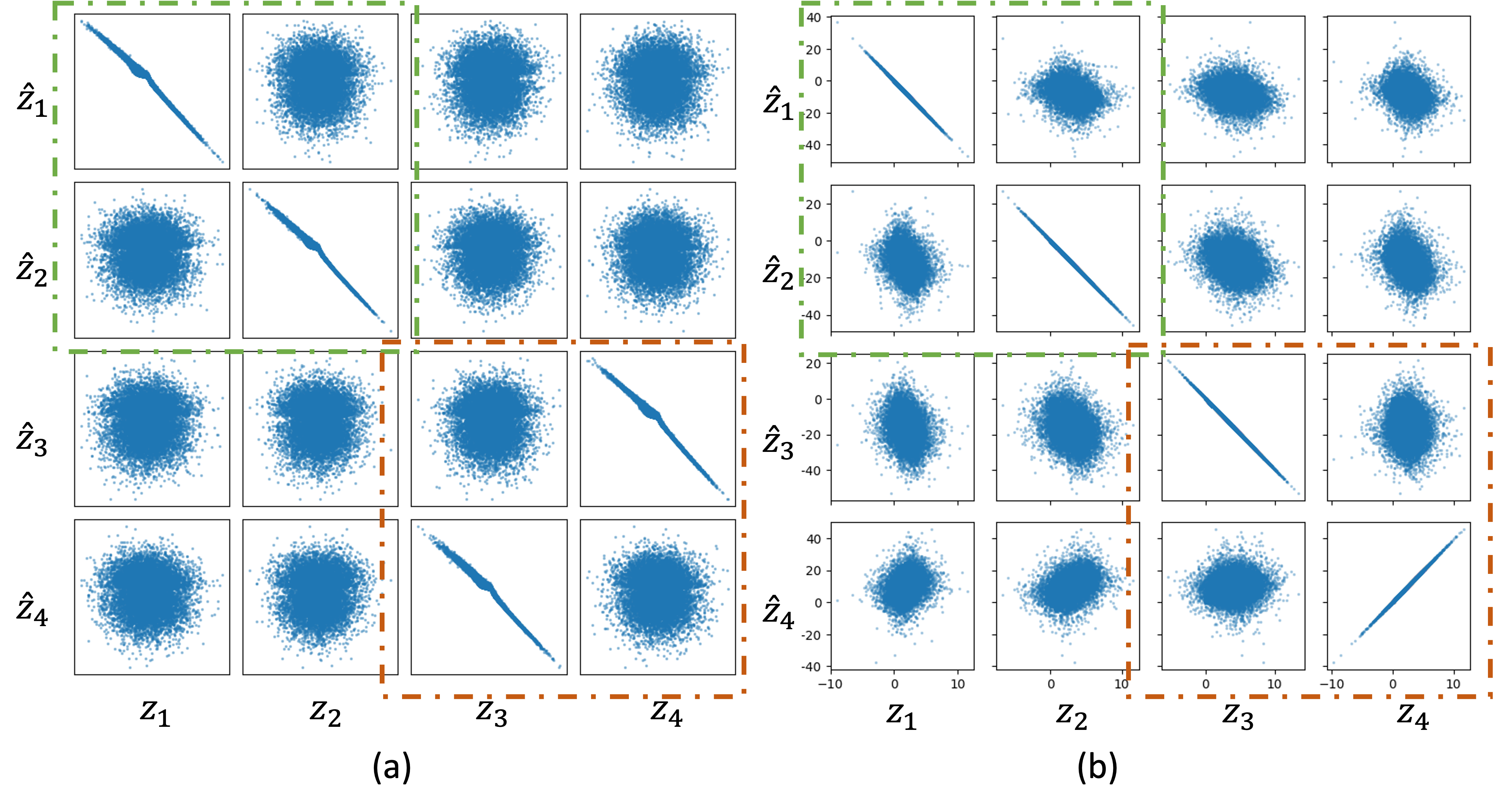}
        \captionof{figure}{Ground-truth versus estimated latents for (a) nonlinear and (b) linear mixing in the two-layer, four-variable setting. Dashed blocks mark layers.}
        \label{scatter}
    \end{minipage}
    \hfill
    \begin{minipage}[t]{0.38\textwidth}
        \vspace{0pt}
        \centering
        \small
        \setlength{\tabcolsep}{4pt}
        \begin{tabular}{llc}
            \toprule
            Type & Layer dimensions & MCC \\
            \midrule
            \multirow{5}{*}{Linear}
              & $[2,2]$         & 0.99 \\
              & $[4,4]$         & 0.99 \\
              & $[8,8,8]$       & 0.99 \\
              & $[16,16,16,16]$ & 0.93 \\
              & $[32,32,32,32]$ & 0.85 \\
            \midrule
            \multirow{5}{*}{Nonlinear}
              & $[2,2]$       & 0.98 \\
              & $[4,4]$       & 0.98 \\
              & $[8,8]$       & 0.98 \\
              & $[3,3,3]$     & 0.97 \\
              & $[4,4,4,4]$   & 0.96 \\
            \bottomrule
        \end{tabular}
        \captionof{table}{Initial latent recovery.}
        \label{mcc table}
    \end{minipage}
\end{figure*}

\paragraph{S-RQ1: Can MoRe identify hierarchical components?}
We first verify the identifiability results empirically. As shown in Table~\ref{mcc table}, MoRe recovers the latents well in both regimes: with linear mixing, the MCC stays at 0.99 up to 24 latent dimensions and decreases to 0.93 and 0.85 at 64 and 128 dimensions, while with nonlinear mixing it ranges from 0.96 to 0.98. Qualitatively, Figure~\ref{scatter} displays clear diagonal alignments between the true and the estimated latents. Notably, as highlighted by the dashed blocks, permutations occur only within layers ($z_1,z_2$ and $z_3,z_4$ are grouped), which is consistent with Theorem~\ref{thm1: latent variables}.

\paragraph{S-RQ2: Does MoRe correctly decide what to reuse and what to add?}
We next examine whether the reuse assessment selects the right modules. With linear mixing, the gate is correct in every run of C1--C7, including the selective cases C4 and C5, which require reuse and expansion at the same time and where the target MCC exceeds 0.999. The same holds for the larger suite C11--C17, in which the recovery of the old library is also unchanged after adaptation.
For nonlinear mixing, Table~\ref{tab:nonlinear_reuse} reports the decoder-matching variant, and reuse remains accurate as long as the mixing is shared, even with noisy observations. Interestingly, when the target observations are rotated independently, which violates the shared-mixing assumption, the gate accuracy drops to 0.50 although the target MCC stays at 0.97. This shows that recovering the latents and matching them across stages are two different problems: the latter requires compatible observation mappings, as discussed in Section~\ref{sec: initial estimation}.

\begin{figure}[t]
    \centering
    \includegraphics[width=\linewidth]{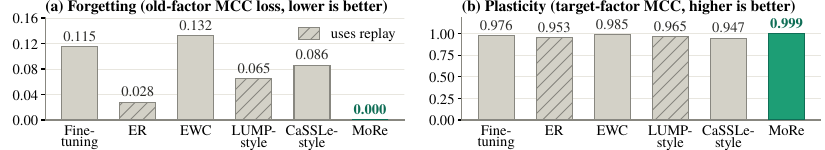}
    \caption{MoRe versus continual-learning baselines on the linear scenarios C1--C7. Forgetting is $\max(0,\text{reference MCC}-\text{post-adaptation MCC})$ on old factors; ER and LUMP-style use 500 replay windows. Scores use the fitting observations and measure factor recovery, not held-out generalization.}
    \label{fig:synthetic_baselines}
\end{figure}

\begin{table}[t]
    \centering
    \small
    \setlength{\tabcolsep}{3pt}
    \begin{minipage}[t]{0.47\linewidth}
        \centering
        \caption{Nonlinear reuse and target recovery.}
        \label{tab:nonlinear_reuse}
        \begin{tabular}{lcc}
            \toprule
            Condition & Gate accuracy\,$\uparrow$ & MCC\,$\uparrow$ \\
            \midrule
            Shared, clean      & 0.984 & 0.977 \\
            Shared, noise 0.02 & 0.984 & 0.969 \\
            Shared, noise 0.05 & 0.992 & 0.964 \\
            Rotated target     & 0.500 & 0.965 \\
            \bottomrule
        \end{tabular}
    \end{minipage}
    \hfill
    \begin{minipage}[t]{0.50\linewidth}
        \centering
        \caption{Dormant recurrence at capacity nine.}
        \label{tab:capacity_adaptation}
        \begin{tabular}{lcc}
            \toprule
            Policy & New MCC\,$\uparrow$ & Forgetting\,$\downarrow$ \\
            \midrule
            Preserve         & $0.559$ & $0.000$ \\
            Selective        & $0.901$ & $0.019$ \\
            Random           & $0.740$ & $0.041$ \\
            Full fine-tuning & $0.780$ & $0.112$ \\
            \bottomrule
        \end{tabular}
    \end{minipage}
\end{table}

\paragraph{S-RQ3: How does MoRe compare with continual-learning baselines?}
To position MoRe among existing approaches, we compare it with fine-tuning, experience replay (ER), regularization (EWC-style~\cite{kirkpatrick2017overcoming}), replay with mixup (LUMP-style~\cite{madaan2021representational}), and distillation (CaSSLe-style~\cite{fini2022self}), each adapted to our temporal objective such that only the continual mechanism differs (Appendix~\ref{app:baselines}). As shown in Figure~\ref{fig:synthetic_baselines}, MoRe exhibits no forgetting relative to its pre-adaptation reference, whereas ER forgets 0.028 and the other baselines between 0.065 and 0.132. At the same time, MoRe achieves the highest target recovery (0.999 versus 0.947--0.985), and its old-factor MCC after adaptation is 0.998 in absolute terms. Note that freezing the model or only expanding it also avoids forgetting by construction; however, these controls recover the target less well (0.930--0.972).

\paragraph{S-RQ4: Is MoRe robust to stage order and long sequences?}
A natural concern is whether MoRe relies on new data being more specific than old data, or on a particular order of stages (Table~\ref{tab:order_long}). In C8, the initial stage observes $z_1$ and $z_3$, while the next stage observes only an unseen fundamental factor $z_2$; MoRe correctly infers a target rank of one and adds a single module, with gate accuracy 1.0 and MCC above 0.999. In C9, $z_2$ and $z_3$ are observed together, and a masked update learns their new-to-old link with a standardized edge-magnitude error of $0.109\pm0.042$, versus $0.217\pm0.029$ for a separate post-hoc fit.
Regarding the order, across all 24 orders of four task families sharing eight primitives, every run ends with exactly eight modules, and the learned representation agrees with the jointly trained one and across orders (CKA above 0.999). Finally, in 32-stage streams with dormant and returning factors, MoRe keeps eight modules with gate accuracy 1.0, whereas naive expansion accumulates 76, at the price of a slightly higher target NMSE ($2.09\times10^{-5}$ versus $6.54\times10^{-7}$).

\paragraph{S-RQ5: What does selective adaptation buy under limited capacity?}
Finally, we study the capacity-constrained regime, in which storage is limited and existing coordinates may be modified, using four types of streams (temporal drift, observation-coordinate drift, new structure, and dormant recurrence), three capacities, and five update policies with matched storage. Under dormant recurrence at capacity nine (Table~\ref{tab:capacity_adaptation}), selective adaptation raises the new-data MCC from 0.559 under preservation to 0.901, with less forgetting than random selection or full fine-tuning, and the gap to preservation closes as the capacity grows. Under temporal drift, it reduces the prediction NMSE from 1.064 to 0.292, close to full fine-tuning (0.266), while forgetting 0.093 instead of 0.259.
Nevertheless, three observations qualify this conclusion. See Appendix~\ref{app:capacity results} for the full results. First, a structure-blind ranking gives nearly the same results as the hierarchy-guided one, so the gain comes from selective adaptation itself. Second, in the new-structure stream, no policy improves the new-data MCC over preservation at capacity nine. Third, adaptation overwrites dormant content: at capacity nine, the MCC on the returning data drops to 0.785 after the update, compared with 0.999 under preservation.

\begin{table}[t]
    \centering
    \small
    \setlength{\tabcolsep}{3pt}
    \begin{minipage}[t]{0.51\linewidth}
        \centering
        \caption{AMPS continual learning (NMSE).}
        \label{tab:amps_continual}
        \begin{tabular}{lcccc}
            \toprule
            Method & Old inc.\,$\downarrow$ & Target\,$\downarrow$ & Mean\,$\downarrow$ & Width \\
            \midrule
            MoRe            & 0.0000 & 0.0341 & 0.0303 & 7.97 \\
            Frozen          & 0.0000 & 0.0379 & 0.0323 & 7.87 \\
            Fine-tuning     & 0.0111 & 0.0265 & 0.0322 & 7.97 \\
            CaSSLe-style    & 0.0069 & 0.0280 & 0.0308 & 7.97 \\
            EWC-style       & 0.0043 & 0.0315 & 0.0312 & 7.97 \\
            Replay (500)    & 0.0030 & 0.0300 & 0.0298 & 7.97 \\
            Naive exp.      & 0.0000 & 0.0226 & 0.0229 & 15.83 \\
            \bottomrule
        \end{tabular}
    \end{minipage}
    \hfill
    \begin{minipage}[t]{0.47\linewidth}
        \centering
        \caption{AMPS linear-probe accuracy (\%).}
        \label{tab:rw_rq2_amps}
        \setlength{\tabcolsep}{2pt}
        \begin{tabular}{lccccc}
            \toprule
            Repr. & Dim & $n$=10 & $n$=50 & $n$=100 & $n$=1k \\
            \midrule
            PCA       & 256 & 77.0 & 84.6 & 84.8 & 84.9 \\
            MoRe\_all & 256 & 79.3 & 84.3 & 84.1 & 85.4 \\
            MoRe\_L2  &  64 & \textbf{80.0} & \textbf{84.5} & 84.6 & 85.7 \\
            MoRe\_L3  &  64 & 79.3 & 83.1 & 83.8 & \textbf{86.1} \\
            \bottomrule
        \end{tabular}
    \end{minipage}
\end{table}

\begin{figure}[t]
    \centering
    \includegraphics[width=0.93\linewidth]{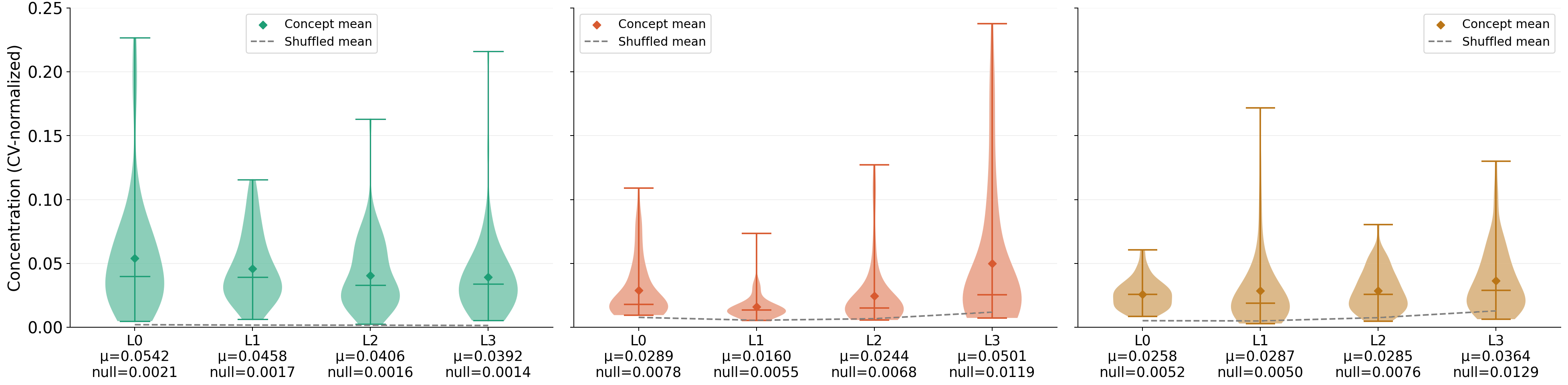}
    \vspace{-0.1cm}
    \caption{Per-layer category concentration (64 coordinates per layer) on AMPS, arXiv, and Wikipedia (left to right). Diamonds: concept means; dashed lines: shuffled-label nulls. arXiv and Wikipedia use held-out passages with dataset-specific PCA; AMPS follows the original protocol.}
    \label{fig: gini}
\end{figure}

\subsection{Real-World LLM Activation Experiments}

\paragraph{Setup.}
For real-world experiments, LLM activations are a natural test case, since they encode rich semantic information that is potentially organized hierarchically. We use Pythia-1.4B~\cite{biderman2023pythia} activations on arXiv, AMPS, and Wikipedia, and fit MoRe with $d=4$ layers on dataset-specific PCA coordinates. As references for the level of abstraction, we use subject categories for arXiv, keyword-based difficulty tiers for AMPS, and article length for Wikipedia. Note that these labels are proxies and are used for evaluation only. See Appendix~\ref{app: real-world} for details.

\paragraph{RW-RQ1: Does MoRe learn continually without forgetting on real data?}
We first examine continual learning on natural data, using all six ordered transitions among three AMPS families. We compare MoRe with freezing, fine-tuning, CaSSLe-style distillation, EWC-style regularization, replay, and naive expansion, adapted as in S-RQ3 (Appendix~\ref{app:baselines}). We report the reconstruction NMSE on held-out passages of the old and the target family (Appendix~\ref{app:metrics}).
As shown in Table~\ref{tab:amps_continual}, MoRe improves the target fit over freezing in all 30 runs, without increasing the old error and without replay, and it achieves the lowest mean error among the replay-free methods of the same width; replay is slightly better, and naive expansion fits better at about twice the active width. On average, MoRe reuses 96.1\% of the saved modules and adds 0.43 modules per transition, and the final libraries obtained from different curriculum orders remain consistent (CKA $0.984\pm0.010$).

\paragraph{RW-RQ2: Do MoRe representations improve label-efficient prediction?}
We then evaluate label efficiency by training linear probes on frozen representations. On AMPS difficulty tiers (Table~\ref{tab:rw_rq2_amps}), every single-layer MoRe representation outperforms PCA at $n=10$. In particular, MoRe\_L2 reaches 80.0\% versus 77.0\% with 64 instead of 256 dimensions, and its accuracy at $n=50$ (84.5\%) matches that of PCA at $n=100$ (84.8\%), although the advantage narrows with more labels. On Wikipedia, MoRe exceeds PCA at every budget, e.g., 55.0\% versus 44.7\% with 10 labeled examples per class. On arXiv, MoRe is on par with PCA (48.1\% versus 48.3\% with 10 examples per class, and 63.8\% versus 65.0\% with 100), although its most specific layer alone matches PCA with a quarter of the dimensions (Appendix~\ref{app:probe tables}).

\paragraph{RW-RQ3: Does MoRe uncover category-selective, broad-to-specific structure?}
Finally, for each module, we measure how concentrated its activation is on a single category and compare it with a shuffled-label null (Figure~\ref{fig: gini}). On all three datasets, every layer mean lies well above its null, by at least $25\times$ on AMPS and about $3$--$6\times$ on arXiv and Wikipedia, which indicates that the modules are category-selective. On AMPS, whose labels encode an explicit curriculum, the mean concentration decreases monotonically from L0 to L3 (0.054, 0.046, 0.041, 0.039), consistent with a broad-to-specific ordering. On arXiv and Wikipedia, in contrast, the pattern is not monotonic, so this ordering appears only where the labels encode an explicit curriculum. Cross-layer links on AMPS further trace curriculum-consistent transitions (Appendix~\ref{app: cases}).

%% file: cr-tex/2_related_work.tex
\section{Related Work}
Supervised CL learns a sequence of labeled tasks by regularizing updates \citep{kirkpatrick2017overcoming,li2016learning,aljundi2018memory}, replaying past information \citep{rebuffi2017icarl}, or allocating capacity across tasks through dedicated columns, paths, masks, experts, adapters, low-rank subspaces, or prompts \citep{rusu2016progressive,fernando2017pathnet,aljundi2017expert,serra2018overcoming,mallya2018packnet,shazeer2017outrageously,chen2023lifelong,yu2024boosting,houlsby2019parameter,hu2022lora,wang2023orthogonal,liang2024inflora,wang2022learning,wang2022dualprompt,smith2023coda}. Unsupervised CL removes labels and task identities and judges methods by representation quality, yet still adapts through replay, distillation, regularization, or branch modification \citep{rao2019continual,hu2022streaming,madaan2021representational,fini2022self,gomez2022continually,tang2024kaizen,liu2025branch}. Modular approaches \citep{wang2024comprehensive,pfeiffer2023modular,veniat2021efficientcontinuallearningmodular,ponti2023combining} instantiate \emph{architectural} modularity, with modules selected by tasks, gates, priors, or heuristics. MoRe instead learns \emph{representational} modularity: its modules are latent components of the data-generating process, made identifiable by time-delayed cross-layer dependencies, which gives a data-grounded criterion for what to reuse and when to expand, without a task oracle. An extended discussion is provided in Appendix~\ref{app:extended-related-work}.

%% file: cr-tex/7_conclusions.tex
\section{Conclusion}
\label{sec:conclusion}

We argued that the central challenge of continual learning is representational: incorporating new data with minimal interference requires a representation whose components reflect intrinsic structure in the data and can be selectively updated. We instantiated this view as MoRe, a framework that learns a hierarchical modular representation with identifiability guarantees from cross-layer time-delayed dependencies, paired with a stage-wise procedure that assesses which saved modules the new data can reuse and expands the library, or selectively adapts it when capacity is limited. Controlled synthetic experiments verify plasticity and stability across diverse transitions, and real-world LLM activations show that hierarchical structure emerges naturally and yields label-efficient representations.

\paragraph{Broader Impact.} As foundational research, MoRe has no direct deployment path. Principled module reuse could lower the compute cost of adapting foundation models, broadening access; the interpretability of the learned hierarchy may also support auditing what a model has internalized. On the cautionary side, continually adapting systems can drift or entrench biases from skewed initial data---risks best mitigated by transparency and versioning of module libraries.

\begin{ack}
We would like to acknowledge the support from NSF Award No.~2229881, AI Institute for Societal Decision Making (AI-SDM), the National Institutes of Health (NIH) under Contract R01HL159805, grants from the MBZUAI-WIS Joint Program, and a gift from Aether AI.
\end{ack}

%% file: cr-tex/8_appendix.tex
{\section*{Appendices for \emph{``MoRe: Modular Representations for Continual Representation Learning on Sequential Data''}}
\addcontentsline{toc}{section}{Appendices}

\startcontents[appendix]

\section*{Appendices Contents}
\printcontents[appendix]{l}{1}{}


\section{Definitions and Proofs}

\subsection{Definitions}
\begin{definition}[(Informal) Intimate Neighbor Set \cite{li2025idol},\cite{zhang2024generalsetting}] The intimate neighbor set of  $V_i \in \mathbf{V}$ contains all $V_j\in\mathbf{V}$ that are adjacent to ${V}_i$ in the Markov network $\mathcal{M}_{\mathbf{V}}$ and also adjacent to all other neighbors of $V_i$. 
\label{def: intimate_set}
\end{definition}

\subsection{Proof for
Theorem~\ref{thm1: latent variables}}
\label{app: proof_thm1}

\latentvariableidentifiability*

Both the true and estimated models have $d$ layers
with $m$ latents per layer.
Their latent domains are $\mathbb{R}^{md}$, and
their mixing functions are invertible onto their
common observation image.
The induced latent transformation is assumed to
be a $C^3$ diffeomorphism.
Both models satisfy the independent-noise and
acyclic within-layer instantaneous structure
specified in the generative model.

The conditions in the theorem are defined as follows.

\begin{itemize}
    \item[A1] (Smooth and Positive Density).
    The conditional density
    $p(\mathbf{Z}_t\mid\mathbf{Z}_{t-1})$
    is positive and three times continuously
    differentiable in the current and past latents
    on $\mathbb{R}^{md}\times\mathbb{R}^{md}$.

    \item[A2] (Sufficient Variability).
    Define
    \begin{align}
        \mathbf{W}_{t,i}^a
        &\triangleq
        \left(
        \frac{
            \partial^2
            \log p(\mathbf{Z}_t\mid\mathbf{Z}_{t-1})
        }{
            \partial Z_{t,i}^a
            \partial Z_{t-1,o}^{l'}
        }
        \right)_{l'\in[d],\,o\in[m]}^{\intercal},
        \\
        \mathbf{U}_{t,i}^a
        &\triangleq
        \left(
        \frac{
            \partial^3
            \log p(\mathbf{Z}_t\mid\mathbf{Z}_{t-1})
        }{
            \partial (Z_{t,i}^a)^2
            \partial Z_{t-1,o}^{l'}
        }
        \right)_{l'\in[d],\,o\in[m]}^{\intercal},
        \\
        \mathbf{V}_{t,i,j}^a
        &\triangleq
        \left(
        \frac{
            \partial^3
            \log p(\mathbf{Z}_t\mid\mathbf{Z}_{t-1})
        }{
            \partial Z_{t,i}^a
            \partial Z_{t,j}^a
            \partial Z_{t-1,o}^{l'}
        }
        \right)_{l'\in[d],\,o\in[m]}^{\intercal}.
    \end{align}
    Here $i\in[m]$, $a\in[d]$, and $i<j$ for
    $\mathbf{V}_{t,i,j}^a$.
    For every fixed value of $\mathbf{Z}_t$,
    the collection
    \[
        \{\mathbf{W}_{t,i}^a,
          \mathbf{U}_{t,i}^a\}_{a,i}
        \cup
        \{\mathbf{V}_{t,i,j}^a\}_{a,\,i<j}
    \]
    is linearly independent as a collection
    of functions of $\mathbf{Z}_{t-1}$.

    \item[A3] (Sparse Latent Process for Each Layer).
    For any latent $Z_{t,i}^a$, either its intimate
    neighbor set is empty, or it has a direct child
    whose intimate neighbor set is empty.
    The intimate neighbor sets are defined in
    Definition~\ref{def: intimate_set} for the
    contextual network used by the applicable
    identification result.

    The component-wise refinement invokes the
    corresponding contextual variability and
    estimation assumptions of
    \cite[Theorem~2, arXiv version~4]{li2025idol},
    including its minimum-edge estimation
    requirement. These additional prerequisites
    are used only for this refinement.

    \item[A2-ind] \textbf{Reduced Sufficient Variability.}
    For every fixed value of $\mathbf{Z}_t$,
    the collection
    \[
        \{\mathbf{W}_{t,i}^a,
          \mathbf{U}_{t,i}^a\}_{a,i}
    \]
    is linearly independent as functions of
    $\mathbf{Z}_{t-1}$.
    This condition replaces A2 when both models
    have mutually independent current latents
    conditional on their respective past.
\end{itemize}

\begin{proof}
Assume the true model
$(g,\{f_i\}_{i=1}^{d},p(\boldsymbol{\epsilon}_t))$
and the estimated model
$(\hat g,\{\hat f_i\}_{i=1}^{d},
p(\hat{\boldsymbol{\epsilon}}_t))$
induce the same distribution over the observations
$\mathbf{X}_t$.
Define the latent transformation
\[
    h\triangleq g^{-1}\circ\hat g,
\]
so that
\[
    \mathbf{Z}_t=h(\hat{\mathbf{Z}}_t).
\]
Let
\[
    \mathbf{H}_t
    \triangleq
    \frac{\partial h(\hat{\mathbf{Z}}_t)}
         {\partial\hat{\mathbf{Z}}_t}
\]
denote its Jacobian matrix.

\paragraph{Step 1: Conditional density transformation.}
Conditioning on $\mathbf{X}_{t-1}$, the
change-of-variables formula gives
\[
    p(\hat{\mathbf{Z}}_t\mid\mathbf{X}_{t-1})
    =
    p(\mathbf{Z}_t\mid\mathbf{X}_{t-1})
    |\det(\mathbf{H}_t)|.
\]
Since
$\mathbf{X}_{t-1}
=g(\mathbf{Z}_{t-1})
=\hat g(\hat{\mathbf{Z}}_{t-1})$
and both mixing functions are invertible,
conditioning on $\mathbf{X}_{t-1}$ is equivalent
to conditioning on either latent past.
Therefore,
\begin{align}
    \log p(
        \hat{\mathbf{Z}}_t
        \mid\hat{\mathbf{Z}}_{t-1})
    &=
    \log p(
        \mathbf{Z}_t
        \mid\mathbf{Z}_{t-1})
    +
    \log|\det(\mathbf{H}_t)|.
    \label{eq:log_conditional}
\end{align}

\paragraph{Step 2: Cross-layer conditional independence.}
By the modular representation assumption and
independent exogenous noise,
\[
    p(\mathbf{Z}_t\mid\mathbf{Z}_{t-1})
    =
    \prod_{a=1}^{d}
    p(\mathbf{Z}_t^a\mid\mathbf{Z}_{t-1}).
\]
The same factorization holds across estimated
layers. Thus, for $b\neq b'$ and any $k,l$,
\[
    \frac{\partial^2}{
        \partial\hat Z_{t,k}^{b}
        \partial\hat Z_{t,l}^{b'}}
    \log p(
        \hat{\mathbf{Z}}_t
        \mid\hat{\mathbf{Z}}_{t-1})
    =0.
\]
Applying this to Eq.~\eqref{eq:log_conditional}
gives
\begin{align}
    &
    \frac{
        \partial^2
        \log p(\mathbf{Z}_t\mid\mathbf{Z}_{t-1})
    }{
        \partial\hat Z_{t,k}^{b}
        \partial\hat Z_{t,l}^{b'}
    }
    +
    \frac{
        \partial^2\log|\det(\mathbf{H}_t)|
    }{
        \partial\hat Z_{t,k}^{b}
        \partial\hat Z_{t,l}^{b'}
    }
    =0.
    \label{eq:second_zero}
\end{align}

\paragraph{Step 3: Chain rule expansion.}
Using the chain rule,
\begin{align}
    &
    \frac{
        \partial^2
        \log p(\mathbf{Z}_t\mid\mathbf{Z}_{t-1})
    }{
        \partial\hat Z_{t,k}^{b}
        \partial\hat Z_{t,l}^{b'}
    }
    \nonumber\\
    &=
    \sum_{i,a}
    \frac{
        \partial
        \log p(\mathbf{Z}_t\mid\mathbf{Z}_{t-1})
    }{\partial Z_{t,i}^{a}}
    \frac{
        \partial^2 Z_{t,i}^{a}
    }{
        \partial\hat Z_{t,k}^{b}
        \partial\hat Z_{t,l}^{b'}
    }
    \nonumber\\
    &\quad+
    \sum_{i,a}\sum_{j,a'}
    \frac{
        \partial^2
        \log p(\mathbf{Z}_t\mid\mathbf{Z}_{t-1})
    }{
        \partial Z_{t,i}^{a}
        \partial Z_{t,j}^{a'}
    }
    \frac{
        \partial Z_{t,i}^{a}
    }{\partial\hat Z_{t,k}^{b}}
    \frac{
        \partial Z_{t,j}^{a'}
    }{\partial\hat Z_{t,l}^{b'}}.
    \label{eq:expanded}
\end{align}
The true-layer factorization makes the mixed
derivatives zero when $a\neq a'$.
For $a=a'$, separate the terms with $i=j$
and combine the terms with $i\neq j$
symmetrically over $i<j$.

\paragraph{Step 4: Differentiation with respect to
lagged latents.}
Since $\mathbf{H}_t$ depends only on
$\hat{\mathbf{Z}}_t$,
\[
    \frac{\partial}{\partial Z_{t-1,o}^{l'}}
    \log|\det(\mathbf{H}_t)|=0.
\]
Differentiate Eq.~\eqref{eq:second_zero}
with respect to all $Z_{t-1,o}^{l'}$ and stack
the resulting equations. Using the definitions
in A2 gives
\begin{align}
    \mathbf{0}
    &=
    \sum_{i,a}
    \frac{
        \partial^2 Z_{t,i}^{a}
    }{
        \partial\hat Z_{t,k}^{b}
        \partial\hat Z_{t,l}^{b'}
    }
    \mathbf{W}_{t,i}^{a}
    \nonumber\\
    &\quad+
    \sum_{i,a}
    \frac{
        \partial Z_{t,i}^{a}
    }{\partial\hat Z_{t,k}^{b}}
    \frac{
        \partial Z_{t,i}^{a}
    }{\partial\hat Z_{t,l}^{b'}}
    \mathbf{U}_{t,i}^{a}
    \nonumber\\
    &\quad+
    \sum_{a}\sum_{i<j}
    \Bigg(
    \frac{
        \partial Z_{t,i}^{a}
    }{\partial\hat Z_{t,k}^{b}}
    \frac{
        \partial Z_{t,j}^{a}
    }{\partial\hat Z_{t,l}^{b'}}
    \nonumber\\
    &\hspace{4.8em}+
    \frac{
        \partial Z_{t,j}^{a}
    }{\partial\hat Z_{t,k}^{b}}
    \frac{
        \partial Z_{t,i}^{a}
    }{\partial\hat Z_{t,l}^{b'}}
    \Bigg)
    \mathbf{V}_{t,i,j}^{a}.
    \label{eq:linear_combo}
\end{align}

\paragraph{Step 5: Apply sufficient variability (A2).}
For fixed current latents, the coefficients in
Eq.~\eqref{eq:linear_combo} do not depend on
the past. By A2, all coefficients must vanish:
\begin{align}
    \frac{
        \partial^2 Z_{t,i}^{a}
    }{
        \partial\hat Z_{t,k}^{b}
        \partial\hat Z_{t,l}^{b'}
    }
    &=0,
    \\
    \frac{
        \partial Z_{t,i}^{a}
    }{\partial\hat Z_{t,k}^{b}}
    \frac{
        \partial Z_{t,i}^{a}
    }{\partial\hat Z_{t,l}^{b'}}
    &=0,
    \label{eq:single_layer_support}
    \\
    \frac{
        \partial Z_{t,i}^{a}
    }{\partial\hat Z_{t,k}^{b}}
    \frac{
        \partial Z_{t,j}^{a}
    }{\partial\hat Z_{t,l}^{b'}}
    +
    \frac{
        \partial Z_{t,j}^{a}
    }{\partial\hat Z_{t,k}^{b}}
    \frac{
        \partial Z_{t,i}^{a}
    }{\partial\hat Z_{t,l}^{b'}}
    &=0,
    \quad i<j.
    \label{eq:shared_layer_support}
\end{align}

\paragraph{Step 6: Conclusion.}
Nonsingularity of $\mathbf{H}_t$ ensures that
every row contains a nonzero entry.
Equation~\eqref{eq:single_layer_support} implies
that the nonzero entries of each row lie within
one estimated layer.

Moreover, two true latents in the same true
layer cannot be associated with different
estimated layers. Otherwise, selecting a
nonzero derivative from each associated layer
would make the first product in
Eq.~\eqref{eq:shared_layer_support} nonzero
and the second zero, a contradiction.
Thus, all true latents in a layer are associated
with the same estimated layer.

Since each layer contains $m$ latents, two true
layers cannot share one estimated layer:
their $2m$ Jacobian rows would be supported on
only $m$ columns, contradicting nonsingularity.
The resulting association is a permutation
of layers.
Continuity and nonsingularity make this
association locally constant, and connectedness
makes it globally constant.
On the product latent domain, $h$ therefore
decomposes into invertible within-layer maps
up to a layer permutation.
Its inverse consequently excludes mixing
between latents from different true layers.

For the component-wise refinement, apply the
within-layer identification result under A3
and its additional contextual and estimation
prerequisites.
This yields component-wise identification
within each recovered layer; this refinement
is not used in the preceding cross-layer
argument.

Finally, consider the special case without
instantaneous dependencies in both models.
All mixed-derivative functions
$\mathbf{V}_{t,i,j}^a$ vanish.
The estimated conditional independence now
makes Step 2 valid for every pair of distinct
estimated latents, including pairs within
the same layer.
Under A2-ind, Steps 4--5 therefore give
\[
    \frac{
        \partial Z_{t,i}^{a}
    }{\partial\hat Z_{t,k}^{b}}
    \frac{
        \partial Z_{t,i}^{a}
    }{\partial\hat Z_{t,l}^{b'}}
    =0
    \qquad
    \text{whenever }(b,k)\neq(b',l).
\]
Each Jacobian row has exactly one nonzero
entry, and nonsingularity gives a permutation
of diagonal form.
Hence the transformation is component-wise
up to permutation, consistent with
\cite[Theorem~1]{yao2022tdrl}.
This proves the special case without invoking
A2 or the additional conditions for A3.
\end{proof}

\subsection{Definitions and Proof of
Theorem~\ref{thm2: latent hierarchy}}
\label{app: proof_thm2}

A time-delayed dependency $u\rightarrow v$ means
\[
    Z_{t,v}\not\!\perp\!\!\!\perp Z_{t-1,u}
    \mid\mathbf{Z}_{t-1,-u},
\]
where $\mathbf{Z}_{t-1,-u}$ contains all past
latents except $u$.

Let $\ell(u)$ denote the true layer of $u$.
In the general case, the block correspondence
proved above defines
$\widehat{\ell}(u)=\pi(\ell(u))$.
In the conditionally independent case,
$\widehat{\ell}(u)$ is the estimated layer
containing $\sigma(u)$.
Compatibility means
\[
    u\rightarrow v
    \Longrightarrow
    \widehat{\ell}(u)\leq\widehat{\ell}(v).
\]

Conditions A4--A5 are:
\begin{itemize}
    \item[A4] Each estimated layer contains exactly
    $m$ latents. A delayed effect from estimated
    layer $a$ to $b$ is permitted only if $a\leq b$.
    Instantaneous effects are permitted only within
    the same estimated layer.

    \item[A5] For every pair of true layers $a<b$
    and every pair of latents $u,v$ with
    $\ell(u)=a$ and $\ell(v)=b$, the transition
    dependence $u\rightarrow v$ is present.
\end{itemize}

\begin{proof}
In the general case, the proof of
Theorem~\ref{thm1: latent variables} gives
\[
    \widehat{\mathbf{Z}}_t^{\pi(a)}
    =T_a(\mathbf{Z}_t^a),
\]
with each $T_a$ invertible.
These transformations preserve conditional
dependencies between entire layers.

If $u\rightarrow v$, with $u$ in layer $a$
and $v$ in another layer $b$, then
\[
    \mathbf{Z}_t^b
    \not\!\perp\!\!\!\perp
    \mathbf{Z}_{t-1}^a
    \mid\mathbf{Z}_{t-1}^{-a}.
\]
Otherwise, the target layer's conditional law
would not depend on any past latent in layer $a$,
contradicting $u\rightarrow v$.
This dependence is preserved under the invertible
block correspondence.

Suppose $\pi(a)>\pi(b)$.
Resolving the acyclic instantaneous equations
within estimated layer $r$ gives
\[
    \widehat{\mathbf{Z}}_t^r
    =
    F_r\!\left(
        \widehat{\mathbf{Z}}_{t-1}^{\leq r},
        \widehat{\boldsymbol{\epsilon}}_t^r
    \right).
\]
Since the current exogenous noise is independent
of the past, A4 excludes dependence on the past
of any later estimated layer.
The preserved dependence from $\pi(a)$ to
$\pi(b)$ is therefore impossible, contradicting
observational equivalence.
Dependencies within a true layer have equal
estimated layer indices.
Thus compatibility holds in the general case.

In the conditionally independent case,
the component-wise correspondence
\[
    \widehat Z_{t,\sigma(u)}
    =\phi_u(Z_{t,u})
\]
preserves the conditional dependencies defining
$u\rightarrow v$.
Placing the matched source after the matched
target again excludes a required dependence
through A4.
Compatibility therefore also holds in this case.

To establish uniqueness under A5, define
\[
    I_a=\{u:\ell(u)=a\},
    \qquad
    J_a=\{u:\widehat{\ell}(u)=a\}.
\]
In both cases, $|I_a|=|J_a|=m$.
If $J_1\neq I_1$, there exist
$v\in J_1\setminus I_1$ and
$u\in I_1\setminus J_1$.
Condition A5 gives $u\rightarrow v$, but
$\widehat{\ell}(u)>1=\widehat{\ell}(v)$,
contradicting compatibility.
Hence $J_1=I_1$.

Assume $J_a=I_a$ for every $a<k$.
If $J_k\neq I_k$, choose
$v\in J_k\setminus I_k$ and
$u\in I_k\setminus J_k$.
The earlier layers have already been matched,
so $\ell(v)>k$ and $\widehat{\ell}(u)>k$.
Condition A5 again gives $u\rightarrow v$,
whereas
\[
    \widehat{\ell}(u)>k=\widehat{\ell}(v),
\]
contradicting compatibility.
Therefore $J_k=I_k$.
Induction establishes the entire ordered partition.
\end{proof}

\section{Implementation Details}

\subsection{Initial Estimation}
\label{app: initial estimation}

\paragraph{Linear instantiation.}
In the linear case, let \(\mathbf{x}_t=\mathbf{V}_0^{\top}(\mathbf{X}_t-\boldsymbol{\mu}_0)\) denote the projected observations and \(\widehat{\mathbf{Z}}_t=\mathbf{F}_0^{\top}\mathbf{x}_t\). The temporal model is linear with innovations \(\mathbf{E}_t=\widehat{\mathbf{Z}}_t-\mathbf{B}\widehat{\mathbf{Z}}_{t-1}\), where \(\mathbf{B}\) is the transition matrix (a VAR of order \(p\) is used for LLM activations). We use Laplace densities for the conditioning latents and for the innovations and profile out the scale of every coordinate, which gives
\begin{equation}
    \mathcal{L}_{\mathrm{profile}}
    =\frac{1}{dm}\sum_{j=1}^{dm}\log\Bigl(\tfrac{1}{T}\sum_{t}\bigl|\widehat{Z}_{t-1,j}\bigr|\Bigr)
    +\frac{1}{dm}\sum_{j=1}^{dm}\log\Bigl(\tfrac{1}{T}\sum_{t}\bigl|E_{t,j}\bigr|\Bigr)
    -\frac{1}{dm}\log\det\bigl(\mathbf{F}_0^{\top}\mathbf{F}_0\bigr).
    \label{eq: profile objective}
\end{equation}
The second and third terms are the conditional negative log-likelihood of Eq.~\eqref{eq: initial objective} for this model, and the first term is the corresponding likelihood of the conditioning latents. The objective is invariant to the rescaling of individual coordinates, which removes a scale degeneracy of the unprofiled likelihood.

\paragraph{Hierarchical regularizer.}
The regularizer \(\mathcal{R}_{\mathrm{h}}\) acts on the time-delayed transition. Entries of \(\mathbf{B}\) from a later layer to an earlier layer are fixed to zero, which implements the hierarchical mask (A4), and an \(\ell_1\) penalty with weight \(\lambda_{\mathrm{h}}\) is applied to the remaining cross-layer entries to promote sparse dependencies.

\paragraph{Nonlinear mixing with \(dm<n\).}
When $dm < n$, the exact likelihood follows from the rectangular change of variables, with volume term \(\tfrac12\log\det(\mathbf{J}_{\mathrm{Dec}_{\phi}}^{\top}\mathbf{J}_{\mathrm{Dec}_{\phi}})\), together with invertibility on the image of the decoder. For computational efficiency, we enforce invertibility with a reconstruction loss and omit the volume term: an encoder--decoder pair is trained with reconstruction and temporal prediction losses. We then apply a \emph{temporal-profile calibration}: an invertible transformation of the latent coordinates is fitted with the profile objective~\eqref{eq: profile objective}, including its log-determinant term, and each coordinate is normalized by its fitted scale. The decoder is composed with the inverse transformation, so reconstructions and predictions are unchanged. This estimator is motivated by the likelihood in Eq.~\eqref{eq: initial objective} but is not an exact implementation of it.

\subsection{Nonlinear Method}
\label{app: nonlinear method}

\paragraph{Constructing the library for the nonlinear case.}

For nonlinear observations, the library instead retains a shared decoder $\mathrm{Dec}_{\phi}$, together with the encoder, normalization, latent calibration, and temporal-model parameters. A coordinate's observation effect is queried from this decoder:
\begin{equation}
    \boldsymbol{\phi}_i(u)
    =
    \mathrm{Dec}_{\phi}(u\mathbf{e}_i)
    -
    \mathrm{Dec}_{\phi}(\mathbf{0}),
\end{equation}
where $\mathbf{e}_i$ is the $i$-th coordinate unit vector. The response is computed when needed from the saved model; it is not a separately learned decoder or a stored historical sample. For a linear decoder, it reduces to
$\boldsymbol{\phi}_i(u)=u(\mathbf{A}_0)_{:,i}$.
For a nonlinear decoder, it provides a reference-context signature for matching and need not capture all interactions with other coordinates.
The nonlinear estimation and calibration procedures are detailed in Appendix~\ref{app: initial estimation}.

\paragraph{Reuse assessment in the nonlinear case.}

We fit a current-data encoder--decoder and select its width
$\widehat r_s$ by validation reconstruction and prediction
errors with a complexity penalty. Following temporal-profile
calibration, we compare coordinate-specific responses of
the saved and incoming decoders:
\begin{align}
    \boldsymbol{\phi}_i(u)
    &=
    \mathrm{Dec}_{\mathrm{old}}(u\mathbf{e}_i)
    -\mathrm{Dec}_{\mathrm{old}}(\mathbf{0}),\\
    \boldsymbol{\chi}_j^{(\alpha)}(u)
    &=
    \mathrm{Dec}_{s}(\alpha u\mathbf{e}_j)
    -\mathrm{Dec}_{s}(\mathbf{0}).
\end{align}
These are queries of shared decoders, not independent
generative functions whose outputs are summed. Let
$\boldsymbol{\Phi}_i$ and $\boldsymbol{\Chi}_j^{(\alpha)}$
stack the responses over a fixed query grid. Compatibility
is measured by
\begin{equation}
    S_{ij}
    =
    \max_{\alpha\in\mathcal{A}}
    \operatorname{clip}_{[0,1]}
    \frac{2\langle\boldsymbol{\Phi}_i,
                      \boldsymbol{\Chi}_j^{(\alpha)}\rangle}
    {\max\!\left(
        \|\boldsymbol{\Phi}_i\|_2^2+
        \|\boldsymbol{\Chi}_j^{(\alpha)}\|_2^2,\epsilon
    \right)},
\end{equation}
where $\mathcal{A}$ accounts for sign and scale ambiguity.
We obtain a partial one-to-one matching by maximizing
$\sum_{ij}P_{ij}(S_{ij}-\tau)$, subject to binary assignments
and at most one match per row and column. Unmatched options
allow modules to be rejected. Candidate matchings are fitted
and selected using current validation fit and complexity
penalties.

\subsection{Nonlinear Preservation-First Learning}
\label{app: nonlinear preservation}

Let $\mathbf{a}$ contain matched coordinates expressed
in old-coordinate units and $\mathbf{b}$ contain unmatched
incoming coordinates. Let $\iota(\mathbf{a})$ embed matched
coordinates into the old latent vector, setting inactive
coordinates to zero. Let $\nu(\mathbf{a},\mathbf{b})$
restore the incoming candidate's coordinate ordering
and matched scales.

We retain the old decoder and define
\begin{align}
    \mathrm{Dec}_{s}^{\mathrm{reuse}}(\mathbf{a},\mathbf{b})
    ={}&
    \mathrm{Dec}_{\mathrm{old}}(\iota(\mathbf{a}))
    +\mathrm{Dec}_{\mathrm{cand}}
       (\nu(\mathbf{a},\mathbf{b})) \nonumber\\
    &-\mathrm{Dec}_{\mathrm{cand}}
       (\nu(\mathbf{a},\mathbf{0})).
\end{align}
Therefore,
\begin{equation}
    \mathrm{Dec}_{s}^{\mathrm{reuse}}(\mathbf{a},\mathbf{0})
    =
    \mathrm{Dec}_{\mathrm{old}}(\iota(\mathbf{a})).
\end{equation}
The correction can model interactions between matched
and new coordinates; it does not assume an additive
decomposition into independent coordinate decoders.
When no coordinate is matched, the incoming decoder
is used directly.

The old model remains frozen. We optimize the target
encoder, decoder correction, and temporal predictor using
\begin{equation}
    \mathcal{L}_{\mathrm{adapt}}
    =
    \mathcal{L}_{\mathrm{rec}}
    +\lambda_{\mathrm{pred}}\mathcal{L}_{\mathrm{pred}}
    +\lambda_{\mathrm{anchor}}
      \mathbb{E}_{\mathbf{X}\sim\mathcal{D}_s}
      \left\|
        \mathrm{Enc}_s(\mathbf{X})
        -\mathbf{z}_{\mathrm{ref}}(\mathbf{X})
      \right\|_2^2,
\end{equation}
where reconstruction and prediction are evaluated in
observation space. The reference coordinates are produced
by the fitted current-data candidate and expressed in
the matched coordinate convention. They are neither
ground-truth latents nor targets computed from historical
observations. The target temporal predictor is learned
rather than constrained to preserve every saved temporal
mechanism.

\paragraph{Nonlinear Capacity-Constrained Adaptation}
\label{app: nonlinear capacity}

The capacity-constrained extension stores a frozen
initial backbone and coordinate-associated residual
networks. Each allocated slot has an encoder residual
mapping observations to a scalar, a decoder residual
mapping the padded latent vector to observations, and
a temporal residual producing a scalar prediction.
The residual networks use one hidden layer with 32
tanh units and zero-initialized final layers.

Selecting a slot permits updates to its residual
networks; unselected residual parameters remain fixed.
Encoder coordinates that are not selected depend only on the frozen
backbone and their own frozen residuals, so their outputs
are unchanged. Decoder and temporal outputs can nevertheless
change indirectly because their latent inputs include
other coordinates. Functional preservation is therefore
evaluated separately from parameter preservation.

All allocated slots, including unused capacity, count
toward persistent storage. The selective policies share
a slot-update quota, whereas full fine-tuning additionally
permits changes to the initial backbone. The comparison
is storage-matched; full fine-tuning is not update-matched.
Selection, fitting, and checkpoint decisions use current
training and validation data, with old observations
reserved for evaluation.

\subsection{Layer-wise Estimation}
\label{app: layer-wise estimation}

We give the implementation details for the progressive estimator used when the joint objective in Eq.~\eqref{eq: initial objective} is too expensive to optimize. Throughout this subsection, \(x_t := \mathbf{V}_0^{\top}(\mathbf{X}_t-\boldsymbol{\mu}_0)\) denotes the projected observation, \(f_i\) the part of the encoder that produces layer \(i\), \(\tilde g_i\) the decoder for the layers recovered so far, and \(J_B\) the estimated transition. Layers are estimated in the order of the generative model: layer \(1\) is the most fundamental, and layer \(i\) is estimated after layers \(1,\ldots,i-1\) have been recovered. Empirical results for this estimator are reported in Appendix~\ref{app: layerwise results}.

\paragraph{Principle.}
The layer-wise estimator follows directly from the hierarchical restriction in Eq.~\eqref{eq: hierarchical_model_layer_z}. Once \(\hat Z^{<i}_t := (\hat Z^1_t,\ldots,\hat Z^{i-1}_t)\) has been learned, these earlier modules are treated as anchors. The next encoder \(f_i\) is trained so that \(\hat Z^i_t=f_i(x_t)\) is (i) informative about the part of \(x_t\) not yet explained by previous modules, (ii) predictive from the allowed delayed parents \((\hat Z^i_{t-1},\hat Z^{<i}_{t-1})\), and (iii) conditionally independent of the raw past observation once these allowed parents are known. The last condition prevents the new layer from hiding information that should have been captured by the already learned layers.

\paragraph{First-layer objective.}
For the first layer, define \(\hat Z^1_t=f_1(x_t)\). We use two conditional density models with matched architecture and noise prior,
\begin{align}
q^0_1(\hat z^1_t \mid \hat z^1_{t-1}),
\qquad
q^1_1(\hat z^1_t \mid \hat z^1_{t-1}, x_{t-1}),
\end{align}
where \(q^0_1\) is the admissible first-layer temporal model and \(q^1_1\) is an augmented model that also sees the raw past. The conditional-independence requirement
\begin{align}
\hat Z^1_t \perp\!\!\!\perp X_{t-1} \mid \hat Z^1_{t-1}
\end{align}
is enforced by the density-ratio estimate
\begin{align}
\mathcal L^1_{\mathrm{cmi}}
=
\mathbb E_t\!
\left[
\log q^1_1(\hat Z^1_t \mid \hat Z^1_{t-1},x_{t-1})
-
\log q^0_1(\hat Z^1_t \mid \hat Z^1_{t-1})
\right].
\label{eq:first-layer-cmi}
\end{align}
Minimizing Eq.~\eqref{eq:first-layer-cmi} discourages any additional predictive information in \(x_{t-1}\) beyond \(\hat Z^1_{t-1}\). To avoid collapse, we combine this term with an informativeness loss and a temporal prediction loss. In our implementation, the informativeness term is a reconstruction loss
\begin{align}
\mathcal L^1_{\mathrm{rec}}
=
\mathbb E_t\left[\lVert x_t-\tilde g_1(\hat Z^1_t)\rVert_2^2\right],
\end{align}
and the prediction term is a JEPA-style loss with an exponential-moving-average target encoder \(\bar f_1\),
\begin{align}
\mathcal L^1_{\mathrm{pred}}
=
\mathbb E_t\left[\lVert r_1(\hat Z^1_{t-1})-\operatorname{sg}(\bar f_1(x_t))\rVert_2^2\right],
\end{align}
where \(r_1\) is a latent predictor and \(\operatorname{sg}(\cdot)\) denotes stop-gradient. A variance floor
\begin{align}
\mathcal L^1_{\mathrm{var}}
=
\sum_{k}\max\{0,\rho-\operatorname{Std}(\hat Z^1_{t,k})\}
\end{align}
is optionally used to rule out degenerate constant representations.

\paragraph{Conditional-density instantiation.}
The density models \(q^0_i\) and \(q^1_i\) can be implemented either as conditional normalizing flows or as Gaussian conditional density heads. For a flow implementation, with context \(c^0_t\), we write
\begin{align}
e_t = T_{\phi}(\hat z^i_t;c^0_t),
\qquad
\log q_{\phi}(\hat z^i_t\mid c^0_t)
=
\log p_E(e_t)+
\log\left\lvert\det \frac{\partial T_{\phi}(\hat z^i_t;c^0_t)}{\partial \hat z^i_t}\right\rvert .
\label{eq:conditional-flow-density}
\end{align}
For the Gaussian-head implementation, \(q_{\phi}(\hat z\mid c)=\mathcal N(\mu_{\phi}(c),\operatorname{diag}\sigma^2_{\phi}(c))\). In both cases, \(q^0_i\) and \(q^1_i\) must share the same density family and noise prior; otherwise, the difference \(\log q^1_i-\log q^0_i\) could reflect extra flexibility rather than additional information in the raw past.

\paragraph{Recursive estimation of later layers.}
After estimating layers \(<i\), we freeze \(f_1,\ldots,f_{i-1}\) and learn only \(f_i\), its decoder head, its predictor, and its density models. Define the admissible context
\begin{align}
c^0_{i,t}:=(\hat Z^i_{t-1},\hat Z^{<i}_{t-1}),
\end{align}
which contains the previous value of the new layer and the delayed values of all already recovered modules. The augmented context is
\begin{align}
c^1_{i,t}:=(\hat Z^i_{t-1},\hat Z^{<i}_{t-1},x_{t-1}).
\end{align}
The layer-\(i\) CMI penalty is then
\begin{align}
\mathcal L^i_{\mathrm{cmi}}
=
\mathbb E_t\!
\left[
\log q^1_i(\hat Z^i_t\mid c^1_{i,t})-
\log q^0_i(\hat Z^i_t\mid c^0_{i,t})
\right],
\label{eq:layer-i-cmi}
\end{align}
which estimates \(I(\hat Z^i_t;X_{t-1}\mid \hat Z^i_{t-1},\hat Z^{<i}_{t-1})\). The reconstruction term is applied to the full representation recovered so far,
\begin{align}
\mathcal L^i_{\mathrm{rec}}
=
\mathbb E_t\left[\lVert x_t-\tilde g_i(\hat Z^{\le i}_t)\rVert_2^2\right],
\qquad
\hat Z^{\le i}_t=(\hat Z^{<i}_t,\hat Z^i_t),
\end{align}
with previous encoders frozen. This allows the decoder to change while preventing later layers from overwriting earlier modules. When desired, one can further residualize the input by replacing \(x_t\) with \(x_t-\tilde g_{i-1}(\hat Z^{<i}_t)\), but our experiments use the full-observation decoder for stability.

The complete layer-\(i\) encoder objective is
\begin{align}
\mathcal L_i
=&\;\lambda_{\mathrm{rec}}\mathcal L^i_{\mathrm{rec}}
+\lambda_{\mathrm{pred}}\mathcal L^i_{\mathrm{pred}}
+\lambda_{\mathrm{cmi}}\mathcal L^i_{\mathrm{cmi}}
+\lambda_{0}\mathcal L^i_{0}
+\lambda_{1}\mathcal L^i_{1}
+\lambda_{\mathrm{var}}\mathcal L^i_{\mathrm{var}}
+\lambda_{\mathrm{orth}}\mathcal L^i_{\mathrm{orth}},
\label{eq:layerwise-objective}
\end{align}
where
\begin{align}
\mathcal L^i_0=-\mathbb E_t\log q^0_i(\hat Z^i_t\mid c^0_{i,t}),
\qquad
\mathcal L^i_1=-\mathbb E_t\log q^1_i(\hat Z^i_t\mid c^1_{i,t}).
\end{align}
The optional decorrelation term
\begin{align}
\mathcal L^i_{\mathrm{orth}}
=\sum_{j<i}\left\lVert \operatorname{Cov}(\hat Z^i_t,\hat Z^j_t)\right\rVert_F^2
\end{align}
helps avoid relearning an already recovered module, but it is not required by the theory because the main hierarchy constraint is imposed through the delayed contexts in Eq.~\eqref{eq:layer-i-cmi}.

\paragraph{Optimization schedule.}
For each layer, we alternate two updates. First, with encoder outputs detached, we fit the density estimators by minimizing \(\mathcal L^i_0+\mathcal L^i_1\). Second, with the density estimators fixed, we update \(f_i\) using Eq.~\eqref{eq:layerwise-objective}. In practice, we warm up the encoder with \(\mathcal L_{\mathrm{pred}}+\mathcal L_{\mathrm{rec}}\) before enabling the CMI penalty. This warm-up prevents early density-estimation noise from collapsing the representation. After warm-up, the prediction loss can be hinge-thresholded at the warm-up value so that the CMI term shapes the representation without sacrificing temporal predictability.

\paragraph{Layer-wise estimator.}
Algorithm~\ref{alg:layerwise-estimation} summarizes the recursive procedure.
\begin{algorithm}[H]
\caption{Layer-wise estimation on the initial dataset}
\label{alg:layerwise-estimation}
\begin{algorithmic}[1]
\Require Initial dataset \(D_0=\{x_t\}_{t=1}^T\), number of layers \(d\), layer width \(m\), lag \(L\), loss weights \(\lambda_{\cdot}\)
\State Fit the projection basis \((\mathbf{V}_0,\boldsymbol{\mu}_0)\) and set \(x_t=\mathbf{V}_0^{\top}(\mathbf{X}_t-\boldsymbol{\mu}_0)\)
\State Initialize the learned module list \(\mathcal M\leftarrow\emptyset\)
\For{\(i=1,\ldots,d\)}
    \State Freeze all encoders in \(\mathcal M\); compute \(\hat Z^{<i}_t\) for all \(t\)
    \State Initialize \(f_i\), decoder \(\tilde g_i\), predictor \(r_i\), target encoder \(\bar f_i\), and density models \((q^0_i,q^1_i)\)
    \For{training epochs}
        \State Form \(\hat Z^i_t=f_i(x_t)\), admissible context \(c^0_{i,t}\), and augmented context \(c^1_{i,t}\)
        \State Update \((q^0_i,q^1_i)\) with \(\hat Z^i_t\) detached by minimizing \(\mathcal L^i_0+\mathcal L^i_1\)
        \State Update \(f_i,\tilde g_i,r_i\) by minimizing Eq.~\eqref{eq:layerwise-objective}
        \State Update the target encoder \(\bar f_i\leftarrow \tau\bar f_i+(1-\tau)f_i\)
    \EndFor
    \State Add \(f_i\) to \(\mathcal M\) and estimate the new transition block from \((\hat Z^{\le i}_{t-1},\hat Z^{\le i}_{t})\)
\EndFor
\State \Return Encoder \(f=(f_1,\ldots,f_d)\), decoder \(\tilde g_d\), transition estimate \(J_B\), and projection basis \((\mathbf{V}_0,\boldsymbol{\mu}_0)\)
\end{algorithmic}
\end{algorithm}

\section{Stage-wise Procedure}
\label{app: adaptation-details}

Algorithm~\ref{alg:adaptation} summarizes one stage of MoRe in the linear case, following Sections~\ref{sec:reuse}--\ref{sec:capacity}. The width of the incoming data is estimated with a relative tolerance \(\varepsilon\) on its singular spectrum, and the compatibility threshold is tied to it as \(\tau=1-10\,\varepsilon\). Only the current data \(\mathcal{D}_s\) and the saved library enter the procedure; no earlier observations, labels, scenario information, or expansion counts are used.

\begin{algorithm}[H]
\caption{MoRe update at stage \(s\) (linear case)}
\label{alg:adaptation}
\begin{algorithmic}[1]
\Require data \(\mathcal{D}_s\); library \(\mathcal{L}_{s-1}\) with decoder matrix \(\mathbf{A}_{s-1}\); tolerance \(\varepsilon\), threshold \(\tau\)
\Require for the capacity-constrained regime: capacity \(C_s\), update quota \(Q_s\), weight \(\beta\)
\Statex \textbf{// Reuse assessment (Section~\ref{sec:reuse})}
\State \(\boldsymbol{\mu}_s\gets\) mean of \(\mathcal{D}_s\); compute the singular values of the centered observations
\State \(\widehat r_s\gets\) number of singular values above the relative tolerance \(\varepsilon\); \(\mathbf{V}_s\gets\) top \(\widehat r_s\) principal directions
\State \(q_i\gets\|\mathbf{V}_s^{\top}\mathbf{a}_{s-1,i}\|_2/\|\mathbf{a}_{s-1,i}\|_2\) for every column \(\mathbf{a}_{s-1,i}\) of \(\mathbf{A}_{s-1}\)
\State \(R_s\gets\{i: q_i\geq\tau\}\), keeping the \(\widehat r_s\) highest scores if more pass; \(d_s\gets\widehat r_s-|R_s|\)
\State Fit Eq.~\eqref{eq: initial objective} on \(\mathcal{D}_s\); match the fitted coordinates one-to-one to \(R_s\) by decoder direction; the unmatched ones initialize the new coordinates
\If{preservation-first (Section~\ref{sec:preservation})}
    \State Learn \(\mathbf{A}_{s,\mathrm{new}}\), the stage-\(s\) encoder, and the permitted transitions of the \(d_s\) new coordinates, with \(\mathbf{A}_{s-1,R_s}\) and the saved transitions fixed
    \State \(\mathbf{A}_s\gets[\mathbf{A}_{s-1},\mathbf{A}_{s,\mathrm{new}}]\); append the stage-\(s\) preprocessing, encoder, and new transition terms to the library
\Else \Comment{capacity-constrained (Section~\ref{sec:capacity})}
    \State \(d_s^{\mathrm{add}}\gets\min\{d_s,\;C_s-|\mathcal{L}_{s-1}|,\;Q_s\}\)
    \State \(\pi_i\gets(1-q_i)+\beta\,\Delta_i^{\mathrm{tem}}\) for every saved coordinate
    \State Select at most \(Q_s-d_s^{\mathrm{add}}\) coordinates with the largest \(\pi_i\), preferring those with fewer descendants
    \State Train the selected coordinates and the \(d_s^{\mathrm{add}}\) new ones on \(\mathcal{D}_s\) with all other parameters frozen
    \State Replace the adapted columns of \(\mathbf{A}_{s-1}\) and append the new ones to obtain \(\mathbf{A}_s\)
\EndIf
\Ensure library \(\mathcal{L}_s\)
\end{algorithmic}
\end{algorithm}

\section{Limitations}
\label{sec:limitations}

\paragraph{Dependence on a shared observation mapping.}
Reuse is decided by comparing saved decoder mappings with incoming data, so it requires the stages to share an observation mapping. When the target observations are rotated independently, the nonlinear gate drops to an accuracy of 0.50 although target recovery remains at 0.97 (Table~\ref{tab:nonlinear_reuse}). A high compatibility score is also a necessary condition for reuse, not a proof that a saved coordinate and an incoming factor are the same.

\paragraph{Capacity-constrained regime.}
Selective adaptation trades retention for acquisition, and the benefit depends on the stream (Appendix~\ref{app:capacity results}). The hierarchy-based preference is not consistently better than a structure-blind ranking. Adapting a coordinate overwrites its previous content: at the smallest capacity, a dormant factor is recovered with an MCC of 0.785 after it returns, compared with 0.999 under preservation. The capacity and update quota are set externally, and MoRe does not predict which inactive coordinates will be needed later. Preservation-first learning avoids these losses but keeps one configuration per stage, so its storage grows with the number of stages.

\paragraph{Real-data evaluation.}
The real-data experiments use activations of a single language model and labels that are proxies (keyword-based difficulty tiers, subject metadata, article length) rather than ground-truth hierarchies. A broad-to-specific ordering of category concentration appears only on AMPS, whose labels encode a curriculum, and the AMPS concentration analysis follows an earlier protocol than the held-out analysis used for arXiv and Wikipedia. Label efficiency exceeds PCA on Wikipedia, exceeds it at small label budgets on AMPS, and is on par with PCA on arXiv. The continual-learning comparison on AMPS covers pairwise and three-stage transitions among three families; with matched width, MoRe is the best replay-free method on mean error, but replay and wider expansion models reach lower error.

\section{Experimental Details}
\label{app: experimental details}

\subsection{Synthetic Experiments}
\label{app: synthetic}

\subsubsection{Data and Transition Scenarios}
\label{app:transitions}

Sequential data are generated from the hierarchical process of Eqs.~\eqref{eq: hierarchical_model_mixing_func} and~\eqref{eq: hierarchical_model_layer_z}. Each transition scenario consists of an initial stage (stage~0) and a second stage (stage~1) that observe different subsets of the latent factors through the same mixing function; Table~\ref{tab:transitions} lists the active factors. Scenarios C1--C9 use two layers of two factors ($z_1,z_2$ in the first layer and $z_3,z_4$ in the second), and C11--C17 use three layers of five factors.
The long-sequence experiments (S-RQ4) use streams built from eight primitives grouped into four task families, visited in all 24 orders, and 32-stage streams in which factors are absent for up to 16 stages. The capacity experiments (S-RQ5) use four-stage streams of four types: temporal drift, observation-coordinate drift, new structure, and dormant recurrence.

\begin{table}[t]
\centering
\small
\setlength{\tabcolsep}{4pt}
\begin{tabular}{llll}
\toprule
\textbf{ID} & \textbf{Name}
& \textbf{Active factors at stage 1} & \textbf{New dims} \\
\midrule
\multicolumn{4}{l}{\textit{Stage 0: $z_1,z_2$}} \\
C1 & No shift       & $z_1,z_2$           & 0 \\
C2 & Full expand    & $z_1,z_2,z_3,z_4$   & 2 \\
C3 & Drop all old   & $z_3,z_4$           & 2 \\
C4$^\star$ & Selective & $z_1,z_3,z_4$    & 2 \\
C5$^\star$ & Selective & $z_2,z_3,z_4$    & 2 \\
C6 & Keep one       & $z_1$               & 0 \\
C7 & Keep one       & $z_2$               & 0 \\
\midrule
\multicolumn{4}{l}{\textit{Stage 0: $z_1,z_3$}} \\
C8 & New fundamental factor & $z_2$       & 1 \\
C9$^\star$ & New parent, retained child & $z_2,z_3$ & 1 \\
\midrule
\multicolumn{4}{l}{\textit{Stage 0: $z_{1:5}$;
three layers of five factors}} \\
C11 & No shift       & $z_{1:5}$             & 0 \\
C12 & Full expand    & $z_{1:15}$            & 10 \\
C13 & Drop all old   & $z_{6:15}$            & 10 \\
C14$^\star$ & Selective & $z_{1:2},z_{6:15}$ & 10 \\
C15$^\star$ & Selective & $z_{3:5},z_{6:15}$ & 10 \\
C16 & Keep two       & $z_{1:2}$             & 0 \\
C17 & Keep three     & $z_{3:5}$             & 0 \\
\bottomrule
\end{tabular}
\caption{Transition scenarios, grouped by the factors active at stage~0. We write $z_{a:b}=\{z_a,\ldots,z_b\}$.
New dimensions count factors present at stage~1 but absent
at stage~0. Starred cases require simultaneous reuse and
expansion. C8 introduces a previously unseen fundamental
factor; C9 additionally retains its child, allowing
estimation of the dependency $z_2\rightarrow z_3$.
For C11--C17, the three layers are $z_{1:5}$,
$z_{6:10}$, and $z_{11:15}$.}
\label{tab:transitions}
\end{table}

\subsubsection{Evaluation Metrics}
\label{app:metrics}

\textbf{MCC.} The mean correlation coefficient between estimated and ground-truth latents after an optimal one-to-one matching. Target (new-data) MCC is computed on the factors active at the current stage, and old-factor MCC on the factors of earlier stages.

\textbf{Forgetting.} The decrease of old-factor MCC relative to its pre-adaptation reference, $\mathrm{MCC}_{\mathrm{old}}^{\mathrm{ref}}-\mathrm{MCC}_{\mathrm{old}}^{\mathrm{after}}$. The baseline comparison in Figure~\ref{fig:synthetic_baselines} reports the clipped value $\max(0,\cdot)$; because a clipped, reference-relative value of zero does not by itself imply perfect recovery, we also report the absolute old-factor MCC after adaptation. In preservation-first experiments, old factors are evaluated with the saved stage configuration; in capacity-constrained experiments, they are evaluated with the current library after modification.

\textbf{Gate accuracy.} The fraction of binary reuse decisions over the saved coordinates that match the ground-truth set of active factors.

\textbf{NMSE.} The normalized mean squared error of reconstruction, or of one-step prediction in the capacity experiments.

\textbf{CKA.} Linear centered kernel alignment between two representations computed on the same evaluation data.

\textbf{Edge-magnitude error (C9).} The absolute difference between the standardized magnitudes of the estimated and the true transition coefficient $z_2\rightarrow z_3$; magnitudes are used because the signs of latent coordinates are arbitrary.

\textbf{Real-data errors.} $E_{\mathrm{old}}$ and $E_{\mathrm{target}}$ denote the reconstruction NMSE on held-out passages of the old and the target family after adaptation. We report the old-error increase $\max(0,E_{\mathrm{old}}-E_{\mathrm{old,ref}})$ and the mean $(E_{\mathrm{old}}+E_{\mathrm{target}})/2$.

\subsubsection{Baselines and Policies}
\label{app:baselines}

\paragraph{Synthetic baseline comparison (S-RQ3).}
All methods use the same linear temporal model, the same profile objective (Appendix~\ref{app: initial estimation}), the inferred target rank as capacity, and the same optimizer and epoch budget; only the continual mechanism differs. \emph{Fine-tuning} continues training on the stage-1 data. \emph{ER} replays 500 stored temporal windows of stage~0. \emph{LUMP-style} mixes current and replayed windows with one mixing coefficient per window, also with 500 windows. \emph{EWC-style} penalizes parameter changes according to their importance under the stage-0 profile loss. \emph{CaSSLe-style} adds a distillation loss toward the normalized stage-0 representation. Each method is run on the seven scenarios with five data seeds and three training seeds (105 runs). Ground-truth latents are used only to compute MCC.
Four further controls are not shown in Figure~\ref{fig:synthetic_baselines}: a frozen stage-0 model (target MCC 0.930, no forgetting), naive expansion (0.967), compositional expansion (0.972), and joint training on both stages (0.951, forgetting 0.010). The comparison is available as aggregate means only, so no error bars are shown.

\paragraph{Capacity policies (S-RQ5).}
\emph{Preserve} never modifies existing coordinates and adds coordinates only within the capacity. \emph{Hierarchy-guided} selection ranks coordinates by Eq.~\eqref{eq: adaptation priority} and prefers those with fewer descendants in the estimated transition graph. \emph{Structure-blind} selection uses the same ranking without the hierarchy preference. \emph{Random} selection chooses the coordinates to update at random. These three selective policies share the same update quota. \emph{Full fine-tuning} may change all allocated coordinates; it is storage-matched but not update-matched. Persistent storage is identical across policies at each capacity.

\paragraph{Real-data baselines (RW-RQ1).}
All methods are evaluated on the same 30 transition--seed cells with the same held-out errors. \emph{Frozen} keeps the stage-0 representation; \emph{fine-tuning} refits it on the target family. \emph{EWC-style}, \emph{CaSSLe-style}, \emph{LUMP-style}, and \emph{MNTDP-style} are linear activation-subspace adaptations of the corresponding mechanisms (subspace regularization, target distillation, latent mixup with replay, and fixed-budget modular selection), not reproductions of the original architectures. \emph{ER} and \emph{LUMP-style} replay 500 stored pairs, and the \emph{joint replay} oracle 16,384. \emph{Naive} and \emph{compositional expansion} enlarge the representation for the target family and use about twice the active width (15.83 modules).

\subsubsection{Settings and Model Selection}
\label{app:hparams}

Reuse assessment has two thresholds, the relative tolerance $\varepsilon$ on the singular spectrum and the compatibility threshold $\tau=1-10\,\varepsilon$ (Appendix~\ref{app: adaptation-details}). The linear synthetic experiments use $\varepsilon=10^{-3}$ and $\tau=0.99$; the AMPS activation experiments use $\varepsilon=0.05$ and $\tau=0.5$, with the same values for all families, orders, and seeds. Linear synthetic sequences have $T=5{,}000$ time steps, the observation dimension is 8 for C1--C7, and models are trained for 2{,}000 epochs with learning rate $5\times10^{-3}$.
The linear reuse experiments use three data seeds for C1--C7, where a larger suite with five data seeds and three training seeds gives the same gate accuracy, and three data seeds with five training seeds for C8, C9, and C11--C17 (15 fits per scenario, 105 fits for C11--C17). The order experiments use three data seeds and three fitting seeds for each of the 24 orders (216 runs), and the long-sequence experiments use nine 32-stage streams.
In the nonlinear experiments, the width, the matching, and the reuse configuration are selected on current-stage validation data by reconstruction and prediction errors with a complexity penalty; each condition in Table~\ref{tab:nonlinear_reuse} comprises three data seeds, three training seeds, and seven scenarios (63 runs).
In the capacity experiments, the capacities are 9, 12, and 15 coordinates, checkpoints are selected on current validation data, and results are averaged over five training seeds within each of three data seeds before the standard deviation across data seeds is computed.
Ground-truth factors, identities, and ranks never enter gates, selection, or checkpointing.

\subsubsection{Stage-Order and Long-Sequence Results}
\label{app:order results}

Table~\ref{tab:order_long} collects the results behind S-RQ4.

\begin{table}[H]
\centering
\small
\setlength{\tabcolsep}{5pt}
\caption{Stage-order and long-sequence results (S-RQ4; mean $\pm$ SD over runs). C8 and C9 use three data seeds and five training seeds; the stage-order experiment uses three data seeds and three fitting seeds for each of the 24 orders. The edge-magnitude error is defined in Appendix~\ref{app:metrics}.}
\label{tab:order_long}
\begin{tabular}{lll}
\toprule
\textbf{Setting} & \textbf{Measure} & \textbf{Result} \\
\midrule
C8: new fundamental factor & Gate accuracy & $1.000\pm0.000$ \\
(15 fits) & Inferred target rank / modules added & 1 / 1 in every fit \\
 & New-factor MCC & $1.000\pm0.000$ \\
 & Forgetting on old factors & $0.000$ \\
\midrule
C9: new parent, retained child & Gate accuracy & $1.000\pm0.000$ \\
(15 fits) & Target MCC & $0.998\pm0.007$ \\
 & Edge-magnitude error, masked joint update & $0.109\pm0.042$ \\
 & Edge-magnitude error, separate post-hoc fit & $0.217\pm0.029$ \\
\midrule
All 24 stage orders & Final modules / true union & 8 / 8 in every run \\
(216 runs) & Sequential-to-joint CKA & $1.000\pm0.000$ \\
 & Sequential-to-joint matched decoder cosine & $0.995\pm0.008$ \\
 & Cross-order CKA & $1.000\pm0.000$ \\
 & Cross-order matched decoder cosine & $0.999\pm0.004$ \\
\midrule
32-stage streams & Final modules, MoRe / naive expansion & 8 / 76 \\
(9 streams) & Gate accuracy & $1.000$ \\
 & Target NMSE, MoRe / naive expansion & $2.09\times10^{-5}$ / $6.54\times10^{-7}$ \\
\bottomrule
\end{tabular}
\end{table}

\subsubsection{Capacity-Constrained Results}
\label{app:capacity results}

Table~\ref{tab:capacity_full} reports all streams, capacities, and policies. Three observations qualify the summary in the main text. First, hierarchy-guided and structure-blind selection give identical results in nine of the twelve stream--capacity cells, and neither is consistently better in the other three, so the benefit over random selection and full fine-tuning comes from compatibility-ranked selective adaptation rather than from the hierarchy preference. Second, the benefit depends on the stream: selective adaptation improves new-data recovery under dormant recurrence and prediction error under the two drift streams, whereas in the new-structure stream no policy improves new-data MCC over preservation at capacity nine. Third, at capacity 15 all selective policies coincide with preservation, except under temporal drift, where the demand is a change of dynamics rather than of width and the results do not depend on capacity.
Table~\ref{tab:capacity_return} shows the cost of overwriting dormant content. At capacity nine, the returning data are represented with an MCC of 0.980 before and 0.785 after the return-stage update under selective adaptation, compared with 0.999 under preservation; at capacity twelve the drop disappears.

\begin{table}[H]
\centering
\scriptsize
\setlength{\tabcolsep}{3pt}
\caption{Capacity-constrained adaptation: all streams, capacities, and policies (means over three data seeds, after averaging five training seeds; standard deviations are below 0.03 in all cells). Forgetting is the loss in old-factor MCC, evaluated with the current library.}
\label{tab:capacity_full}
\begin{tabular}{lccc@{\hskip 8pt}ccc@{\hskip 8pt}ccc}
\toprule
 & \multicolumn{3}{c}{Capacity 9} & \multicolumn{3}{c}{Capacity 12} & \multicolumn{3}{c}{Capacity 15} \\
\cmidrule(lr){2-4}\cmidrule(lr){5-7}\cmidrule(lr){8-10}
Policy & New MCC\,$\uparrow$ & Pred.\ NMSE\,$\downarrow$ & Forget.\,$\downarrow$ & New MCC & Pred.\ NMSE & Forget. & New MCC & Pred.\ NMSE & Forget. \\
\midrule
\multicolumn{10}{l}{\textit{Dormant recurrence}} \\
Preserve & 0.559 & 0.971 & 0.000 & 0.843 & 0.742 & 0.000 & 0.986 & 0.552 & 0.000 \\
Hierarchy-guided & 0.901 & 0.387 & 0.019 & 0.950 & 0.519 & 0.004 & 0.986 & 0.552 & 0.000 \\
Structure-blind & 0.901 & 0.387 & 0.019 & 0.952 & 0.518 & 0.004 & 0.986 & 0.552 & 0.000 \\
Random & 0.740 & 0.539 & 0.041 & 0.918 & 0.548 & 0.014 & 0.986 & 0.552 & 0.000 \\
Full fine-tuning & 0.780 & 0.378 & 0.112 & 0.886 & 0.378 & 0.035 & 0.923 & 0.378 & 0.021 \\
\midrule
\multicolumn{10}{l}{\textit{Temporal drift}} \\
Preserve & 0.999 & 1.064 & 0.000 & 0.999 & 1.064 & 0.000 & 0.999 & 1.064 & 0.000 \\
Hierarchy-guided & 0.950 & 0.292 & 0.093 & 0.950 & 0.292 & 0.093 & 0.950 & 0.292 & 0.093 \\
Structure-blind & 0.950 & 0.292 & 0.093 & 0.950 & 0.292 & 0.093 & 0.950 & 0.292 & 0.093 \\
Random & 0.896 & 0.393 & 0.159 & 0.896 & 0.393 & 0.159 & 0.896 & 0.393 & 0.159 \\
Full fine-tuning & 0.818 & 0.266 & 0.259 & 0.818 & 0.266 & 0.259 & 0.818 & 0.266 & 0.259 \\
\midrule
\multicolumn{10}{l}{\textit{Observation-coordinate drift}} \\
Preserve & 0.999 & 0.526 & 0.000 & 0.999 & 0.371 & 0.000 & 0.999 & 0.336 & 0.000 \\
Hierarchy-guided & 0.997 & 0.318 & 0.001 & 0.999 & 0.319 & 0.000 & 0.999 & 0.336 & 0.000 \\
Structure-blind & 0.997 & 0.318 & 0.001 & 0.999 & 0.319 & 0.000 & 0.999 & 0.336 & 0.000 \\
Random & 0.981 & 0.388 & 0.015 & 0.999 & 0.345 & 0.001 & 0.999 & 0.336 & 0.000 \\
Full fine-tuning & 0.984 & 0.318 & 0.012 & 0.994 & 0.318 & 0.004 & 0.996 & 0.318 & 0.002 \\
\midrule
\multicolumn{10}{l}{\textit{New structure}} \\
Preserve & 0.674 & 0.523 & 0.000 & 0.890 & 0.389 & 0.000 & 0.988 & 0.329 & 0.000 \\
Hierarchy-guided & 0.668 & 0.524 & 0.005 & 0.874 & 0.359 & 0.000 & 0.988 & 0.329 & 0.000 \\
Structure-blind & 0.649 & 0.448 & 0.008 & 0.874 & 0.358 & 0.000 & 0.988 & 0.329 & 0.000 \\
Random & 0.649 & 0.457 & 0.007 & 0.883 & 0.372 & 0.002 & 0.988 & 0.329 & 0.000 \\
Full fine-tuning & 0.628 & 0.408 & 0.031 & 0.844 & 0.354 & 0.010 & 0.943 & 0.330 & 0.003 \\
\bottomrule
\end{tabular}
\end{table}

\begin{table}[H]
\centering
\footnotesize
\setlength{\tabcolsep}{3pt}
\caption{Dormant recurrence: MCC on the returning data before and after the return-stage update (mean\,$\pm$\,SD over data seeds).}
\label{tab:capacity_return}
\begin{tabular}{lcccccc}
\toprule
 & \multicolumn{2}{c}{Capacity 9} & \multicolumn{2}{c}{Capacity 12} & \multicolumn{2}{c}{Capacity 15} \\
\cmidrule(lr){2-3}\cmidrule(lr){4-5}\cmidrule(lr){6-7}
Policy & Before & After & Before & After & Before & After \\
\midrule
Preserve & $0.999\pm0.000$ & $0.999\pm0.000$ & $0.999\pm0.000$ & $0.999\pm0.000$ & $0.999\pm0.000$ & $0.999\pm0.000$ \\
Hierarchy-guided & $0.980\pm0.003$ & $0.785\pm0.008$ & $0.995\pm0.002$ & $0.995\pm0.002$ & $0.999\pm0.000$ & $0.999\pm0.000$ \\
Structure-blind & $0.980\pm0.003$ & $0.785\pm0.008$ & $0.996\pm0.001$ & $0.996\pm0.001$ & $0.999\pm0.000$ & $0.999\pm0.000$ \\
Random & $0.959\pm0.011$ & $0.775\pm0.038$ & $0.985\pm0.007$ & $0.881\pm0.020$ & $0.999\pm0.000$ & $0.961\pm0.016$ \\
Full fine-tuning & $0.885\pm0.003$ & $0.667\pm0.008$ & $0.963\pm0.002$ & $0.901\pm0.011$ & $0.978\pm0.001$ & $0.968\pm0.002$ \\
\bottomrule
\end{tabular}
\end{table}

\subsubsection{Layer-wise Estimation Results}
\label{app: layerwise results}

We focus especially on first-layer estimation because, in the layer-wise estimator, accurate recovery of the initial layer provides the anchor for estimating subsequent layers. Once an earlier layer has been recovered, the same estimation principle can be applied recursively to later layers by conditioning on the already learned lower-level structure. The experiments below therefore test whether the learned scalar representation preferentially aligns with ground-truth first-layer coordinates, and whether this behavior persists when the generator becomes nonlinear and when the hierarchy is extended to three layers.

We next visualize first-layer estimation on synthetic temporal data. In all panels, the horizontal axis is the learned scalar first-layer representation, while the vertical axis is a ground-truth latent coordinate from the synthetic generator. Because recovered latent coordinates are identifiable only up to the usual component-wise indeterminacies, and up to sign and scaling in the scalar linear diagnostics considered here, we report absolute correlations; a strong negative correlation is therefore as meaningful as a strong positive one. The main diagnostic question is whether the learned representation aligns more strongly with first-layer latents than with deeper-layer latents.

Figure~\ref{fig:synth_linear_nongaussian} shows the two-layer linear non-Gaussian case. The learned representation exhibits a strong monotone relationship with the true first-layer latent and a substantially weaker relationship with the second-layer latent. Quantitatively, the final absolute correlation with the first-layer latent is about $0.90$, versus about $0.41$ for the second-layer latent. This is the expected qualitative pattern for successful first-layer recovery: the representation tracks the first layer closely while remaining relatively insensitive to the other latent coordinate.

\begin{figure}[t]
    \centering
    \includegraphics[width=\linewidth]{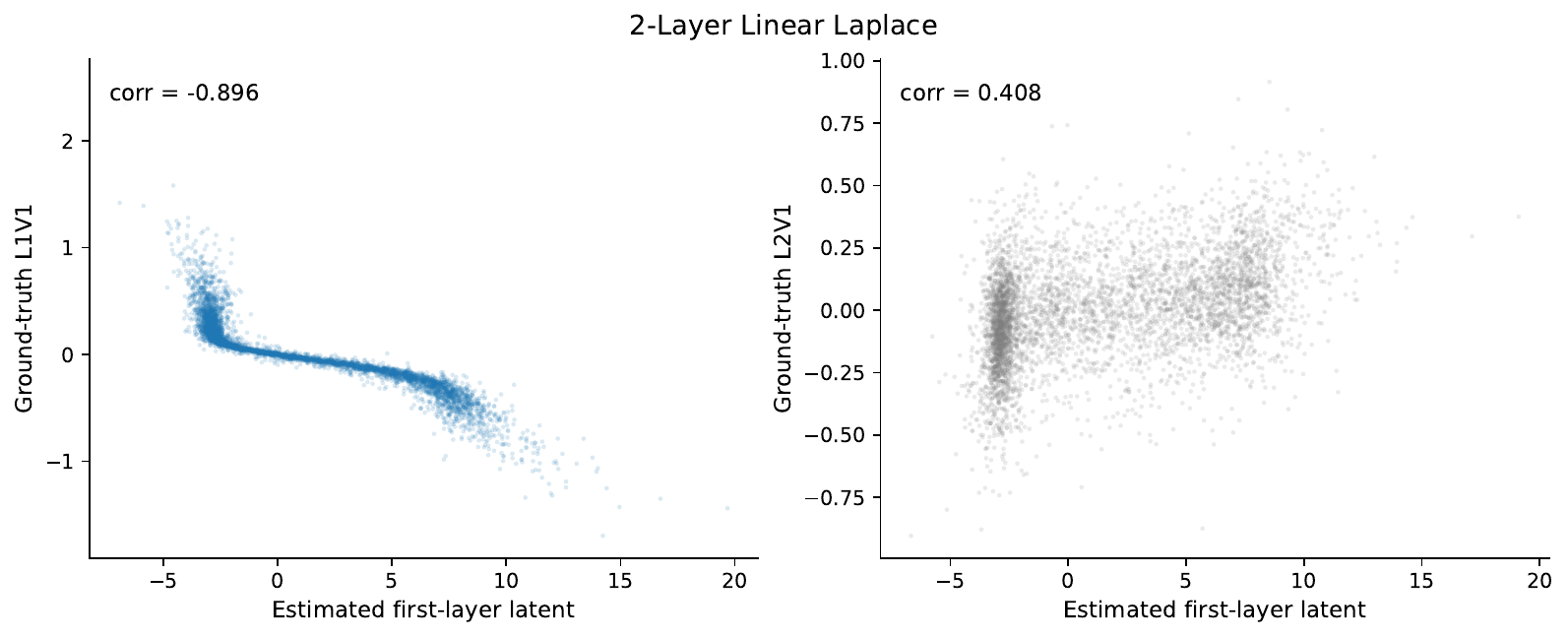}
    \caption{Two-layer linear non-Gaussian synthetic experiment. Each panel plots the learned scalar first-layer representation against one ground-truth latent coordinate. The left panel shows strong alignment with the first-layer latent, while the right panel shows a noticeably weaker relationship with the second-layer latent.}
    \label{fig:synth_linear_nongaussian}
\end{figure}

Figure~\ref{fig:synth_nonlinear} reports the corresponding two-layer nonlinear setting. The same qualitative pattern persists, but the first-layer alignment becomes even stronger: the final absolute correlation with the first-layer latent is about $0.97$, while the second-layer correlation remains around $0.42$. This indicates that the method is not restricted to the linear case and can still isolate the first layer when the latent dynamics and/or observation mapping become nonlinear.

\begin{figure}[t]
    \centering
    \includegraphics[width=\linewidth]{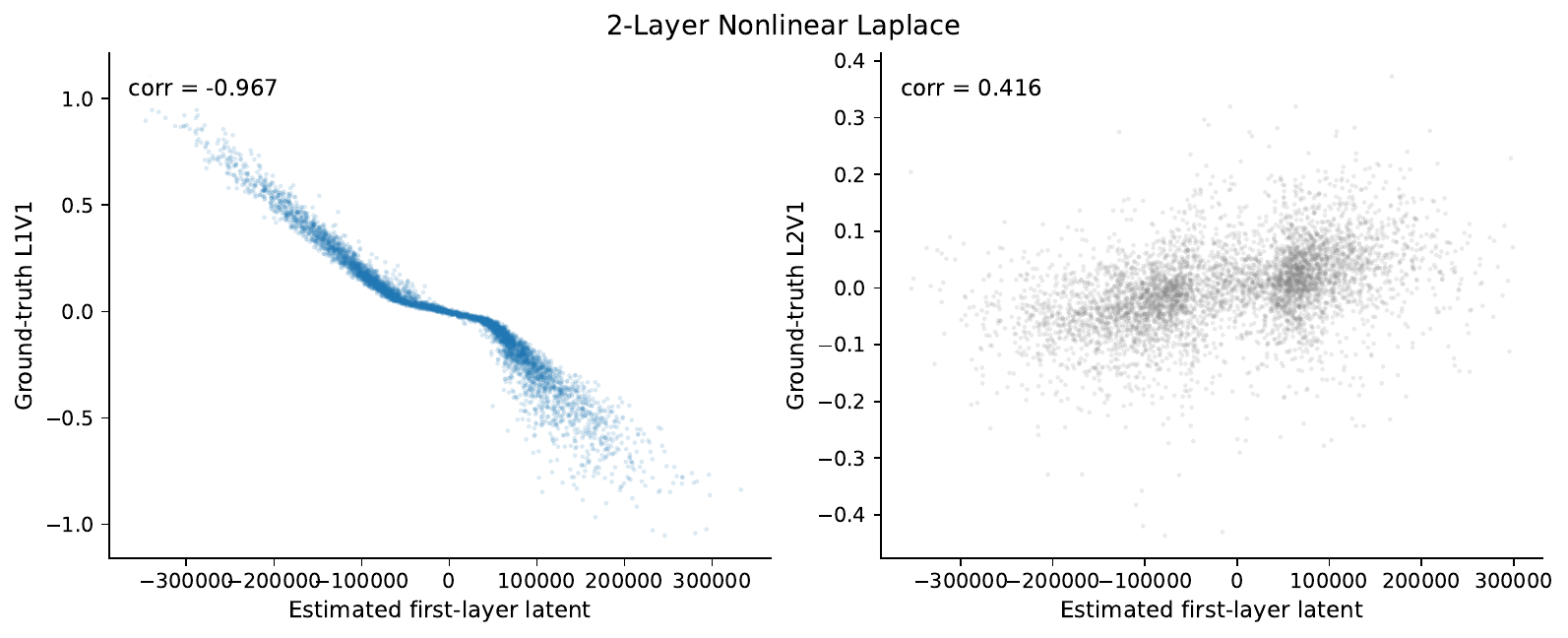}
    \caption{Two-layer nonlinear synthetic experiment. The learned representation remains strongly aligned with the first-layer latent and only weakly aligned with the second-layer latent, showing that the first-layer recovery pattern persists beyond the linear setting.}
    \label{fig:synth_nonlinear}
\end{figure}

For the more challenging three-layer setting with two latent variables per layer, Figures~\ref{fig:synth_three_layer_focus} and~\ref{fig:synth_three_layer_grid} provide two complementary views. Figure~\ref{fig:synth_three_layer_focus} concentrates on the two first-layer coordinates. One coordinate is recovered very strongly ($|\mathrm{corr}| \approx 0.94$), while the other is recovered more moderately ($|\mathrm{corr}| \approx 0.49$). This is consistent with the current estimator being one-dimensional: it captures one dominant direction in the first-layer subspace, but it does not reconstruct the full two-dimensional first-layer subspace.

\begin{figure}[t]
    \centering
    \includegraphics[width=\linewidth]{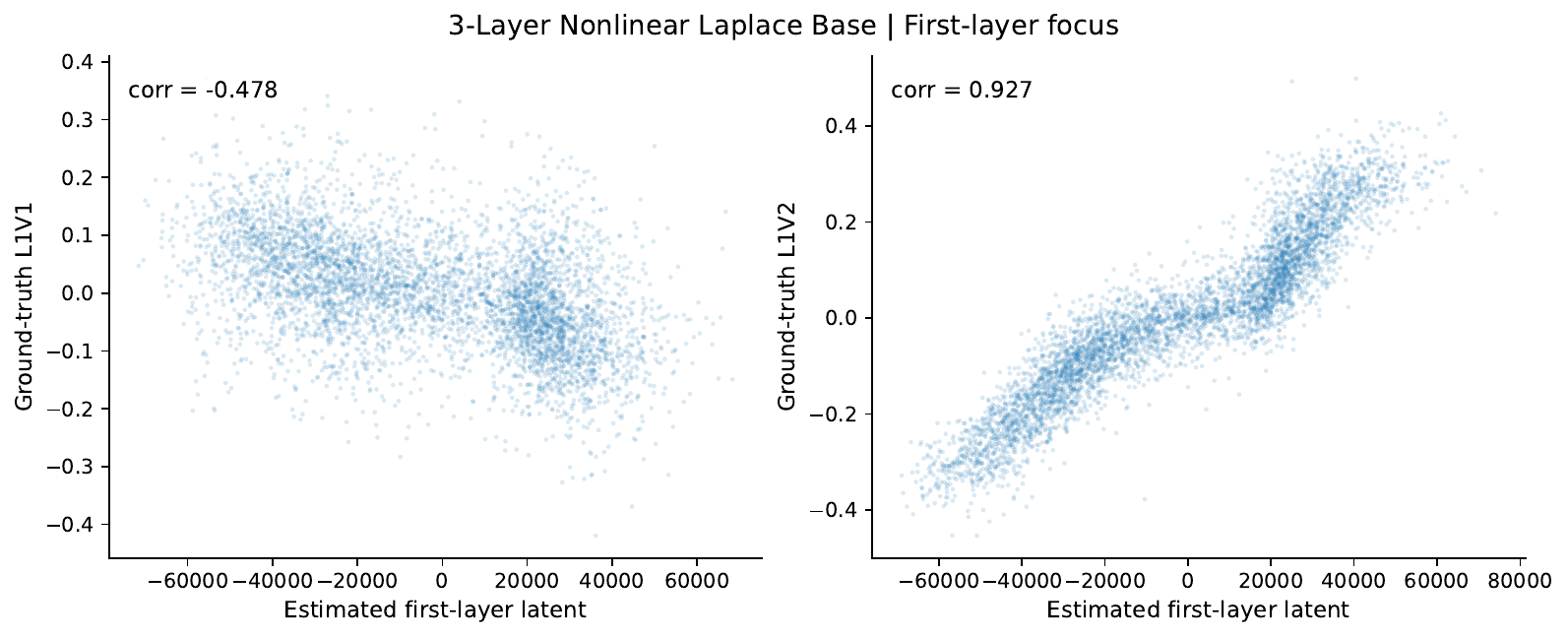}
    \caption{Three-layer synthetic experiment, first-layer focus. The learned scalar representation is plotted against the two first-layer ground-truth coordinates. The estimator strongly captures one dominant first-layer direction, while the second first-layer coordinate is represented more weakly.}
    \label{fig:synth_three_layer_focus}
\end{figure}

Figure~\ref{fig:synth_three_layer_grid} expands the same comparison to all six latent coordinates. The dominant first-layer coordinate still yields the largest absolute correlation, which is the key sign of first-layer recovery. At the same time, some non-first-layer coordinates remain substantially correlated, especially one coordinate in the third layer. We therefore interpret the three-layer result as a partial but meaningful success: the estimator continues to prioritize the first layer, but the separation from the remaining latent coordinates is no longer as clean as in the simpler two-layer case.

\begin{figure*}[t]
    \centering
    \includegraphics[width=\linewidth]{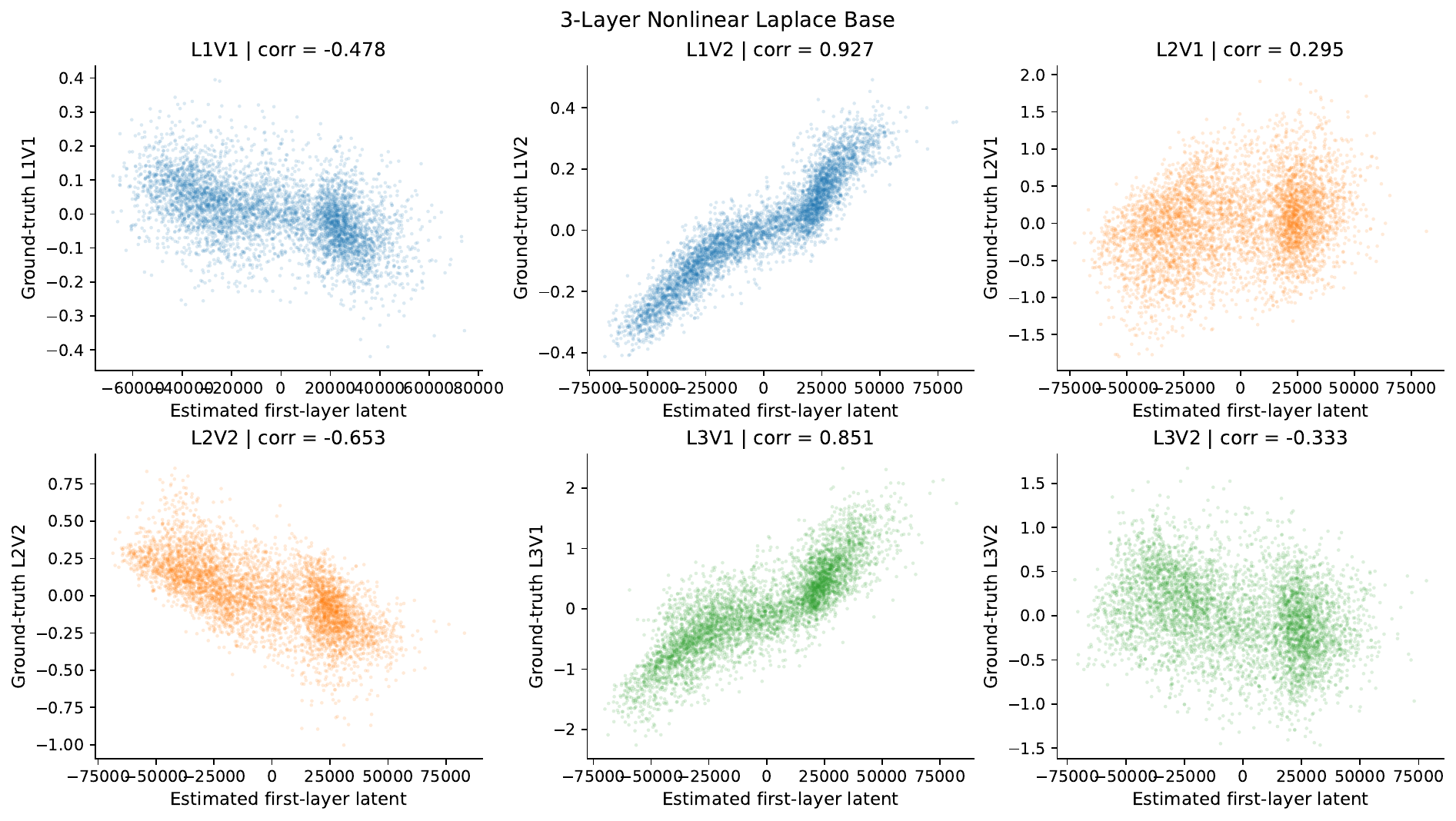}
    \caption{Three-layer synthetic experiment, full cross-layer comparison. Each panel compares the learned scalar first-layer representation with one ground-truth latent coordinate among the six coordinates spanning all three layers. The strongest alignment is attained on a first-layer coordinate, while several non-first-layer coordinates still exhibit sizable, though smaller, correlations.}
    \label{fig:synth_three_layer_grid}
\end{figure*}

Overall, these four figures support a cautious version of the intended conclusion. In the two-layer settings, the learned scalar representation aligns much more strongly with the first-layer latent than with the second-layer latent, under both linear Laplace and nonlinear Laplace data generation. In the harder three-layer setting, the largest absolute correlation is still attained by a first-layer coordinate, but one third-layer coordinate and one second-layer coordinate remain substantially correlated. Thus, the three-layer experiment provides evidence of preferential first-layer estimation in a deeper hierarchy, rather than full separation of all layers or full recovery of the two-dimensional first-layer subspace.

\subsection{Real-world Experiments on LLM Activations}
\label{app: real-world}

\subsubsection{Preprocessing of LLM Activations}

We extract hidden-state activations from the pretrained language model Pythia-1.4B by registering forward hooks at a single intermediate layer (layer 16), chosen to capture rich semantic representations at approximately 60\% of network depth. For each dataset\footnote{\url{https://www.kaggle.com/datasets/Cornell-University/arxiv}, \url{https://huggingface.co/datasets/sarahpann/AMPS}, \url{https://huggingface.co/datasets/wikimedia/wikipedia}}, we stream passages (arXiv title+abstract concatenations, AMPS problem+solution pairs, Wikipedia articles; up to 50,000 passages for AMPS, and 20,000 passages with an 80/20 passage split for arXiv and Wikipedia), tokenize and truncate to 512 tokens, and run the frozen LLM to extract per-token hidden states, i.e., activations. The token-level activations have 2048 dimensions and are compressed to 256 dimensions via Incremental PCA, fitted separately for each dataset on its training passages. Consecutive token windows of length $p+1$ are constructed within passage boundaries to form training samples for MoRe, preserving local sequential structure without crossing document boundaries. We set the number of hierarchy layers to $d=4$ with 64 coordinates per layer. For the VAR model order $p$, we evaluate $p \in \{10, 50, 100\}$: order 10 captures short-range token-level dependencies where cross-layer transitions are strongest (transition weights 5--7$\times$ larger than at $p=100$ on AMPS), order 50 captures mid-range discourse structure, and order 100 serves as an upper bound confirming that the signal degrades at longer lags---consistent with hierarchical concept transitions in structured text being primarily a local phenomenon. For the inter-layer sparsity penalty $\lambda_\text{sp} \in \{10^{-4}, 10^{-5}\}$, both values produce consistent concentration profiles on AMPS, so the results are not sensitive to the precise sparsity level within this range. Unless otherwise stated, we report results at $p=10$ and $\lambda_\text{sp}=10^{-4}$, which yields the strongest cross-layer transition weights.

\subsubsection{Continual Learning on AMPS}
\label{app:amps continual}

Stage~0 fits the profile model on one AMPS family (arithmetic, algebra, or advanced mathematics). At adaptation, MoRe receives only the saved state and the target activation windows; it is not given old examples, family labels, a target rank, or an expansion width. All six ordered pairs are run with five seeds (30 runs), training data are used for fitting, and the reported errors use held-out passages. MoRe infers ranks between 7 and 9 and adds zero, one, or two modules in 18, 11, and 1 of the 30 runs; it stores 8.30 modules on average, of which 7.97 are active. Table~\ref{tab:amps_full} reports all methods and Table~\ref{tab:amps_direction} the reuse by transition direction: transitions from arithmetic toward more advanced material add the most modules, and reverse transitions mostly reuse the saved library.
The cross-order CKA in the main text is computed for four three-stage curricula with five seeds, on identical held-out passages for all six pairs of orders within each seed (30 comparisons); the first adaptation retains 97.0\% of the library and adds 0.5 modules on average, and the second retains 94.2\% and adds 0.3.

\begin{table}[H]
\centering
\small
\setlength{\tabcolsep}{5pt}
\caption{AMPS continual learning, all methods (means over 30 runs; reconstruction NMSE on held-out passages).}
\label{tab:amps_full}
\begin{tabular}{lccccc}
\toprule
Method & Old-error increase\,$\downarrow$ & Target\,$\downarrow$ & Mean\,$\downarrow$ & Active modules & Replay pairs \\
\midrule
Compositional expansion & 0.00000 & 0.02037 & 0.02222 & 15.83 & 0 \\
Naive expansion         & 0.00000 & 0.02255 & 0.02291 & 15.83 & 0 \\
Joint replay (oracle)   & 0.00291 & 0.02958 & 0.02960 & 7.97 & 16{,}384 \\
LUMP-style              & 0.00311 & 0.02974 & 0.02978 & 7.97 & 500 \\
ER                      & 0.00298 & 0.02997 & 0.02983 & 7.97 & 500 \\
MoRe                    & 0.00000 & 0.03410 & 0.03026 & 7.97 & 0 \\
CaSSLe-style            & 0.00690 & 0.02801 & 0.03081 & 7.97 & 0 \\
EWC-style               & 0.00426 & 0.03147 & 0.03122 & 7.97 & 0 \\
Fine-tuning             & 0.01108 & 0.02654 & 0.03216 & 7.97 & 0 \\
MNTDP-style             & 0.01108 & 0.02654 & 0.03216 & 7.97 & 0 \\
Frozen                  & 0.00000 & 0.03793 & 0.03232 & 7.87 & 0 \\
\bottomrule
\end{tabular}
\end{table}

\begin{table}[H]
\centering
\small
\setlength{\tabcolsep}{5pt}
\caption{AMPS reuse by transition direction (means over five seeds).}
\label{tab:amps_direction}
\begin{tabular}{lccc}
\toprule
Transition & Reused modules & New modules & Reused fraction \\
\midrule
arithmetic $\to$ algebra    & 7.4 & 0.6 & 1.000 \\
arithmetic $\to$ advanced   & 7.4 & 0.8 & 1.000 \\
algebra $\to$ arithmetic    & 7.4 & 0.4 & 0.925 \\
algebra $\to$ advanced      & 7.8 & 0.4 & 0.975 \\
advanced $\to$ arithmetic   & 7.4 & 0.2 & 0.908 \\
advanced $\to$ algebra      & 7.8 & 0.2 & 0.956 \\
\bottomrule
\end{tabular}
\end{table}

\subsubsection{Label-Efficiency Tables}
\label{app:probe tables}

Table~\ref{tab:probe_amps_full} gives the AMPS probe results for all layers, and Table~\ref{tab:probe_arxiv_wiki} the arXiv and Wikipedia results. All probes are linear and are trained on frozen representations.

\begin{table}[H]
\centering
\small
\setlength{\tabcolsep}{4pt}
\caption{Linear-probe accuracy (\%) on AMPS tier prediction for all layers; $n$ is the number of labeled passages.}
\label{tab:probe_amps_full}
\begin{tabular}{lcccccc}
\toprule
Repr. & Dim & $n$=10 & $n$=25 & $n$=50 & $n$=100 & $n$=1000 \\
\midrule
PCA        & 256 & 77.0$\pm$2.9 & 83.0$\pm$2.1 & 84.6$\pm$1.5 & 84.8$\pm$0.8 & 84.9$\pm$0.6 \\
\midrule
MoRe\_all  & 256 & 79.3$\pm$4.0 & 83.1$\pm$1.1 & 84.3$\pm$1.1 & 84.1$\pm$0.8 & 85.4$\pm$0.3 \\
MoRe\_L0   &  64 & 79.0$\pm$4.0 & 82.5$\pm$0.9 & 84.2$\pm$1.7 & 84.5$\pm$0.4 & 86.0$\pm$0.3 \\
MoRe\_L1   &  64 & 79.4$\pm$3.6 & 82.8$\pm$1.0 & 84.4$\pm$0.9 & 84.4$\pm$0.8 & 86.0$\pm$0.5 \\
MoRe\_L2   &  64 & 80.0$\pm$4.1 & 83.0$\pm$1.1 & 84.5$\pm$0.8 & 84.6$\pm$1.3 & 85.7$\pm$0.3 \\
MoRe\_L3   &  64 & 79.3$\pm$3.6 & 82.4$\pm$1.2 & 83.1$\pm$1.6 & 83.8$\pm$0.8 & 86.1$\pm$0.4 \\
\bottomrule
\end{tabular}
\end{table}

\begin{table}[H]
\centering
\small
\setlength{\tabcolsep}{4pt}
\caption{Linear-probe accuracy (\%) on arXiv and Wikipedia with dataset-specific PCA (mean\,$\pm$\,SD over five support samples). Columns give the number of labeled examples per class.}
\label{tab:probe_arxiv_wiki}
\begin{tabular}{llcccccc}
\toprule
Dataset & Repr. & Dim & 5 & 10 & 25 & 50 & 100 \\
\midrule
\multirow{6}{*}{arXiv} & PCA & 256 & $40.2\pm1.8$ & $48.3\pm3.7$ & $58.2\pm1.5$ & $63.4\pm1.6$ & $65.0\pm1.4$ \\
 & MoRe\_all & 256 & $39.6\pm2.4$ & $48.1\pm3.4$ & $56.6\pm1.9$ & $61.4\pm1.9$ & $63.8\pm1.8$ \\
 & MoRe\_L0 & 64 & $24.6\pm3.3$ & $30.1\pm3.5$ & $35.5\pm1.3$ & $39.7\pm1.0$ & $44.9\pm0.9$ \\
 & MoRe\_L1 & 64 & $23.1\pm3.1$ & $27.8\pm2.4$ & $32.1\pm1.2$ & $35.7\pm2.0$ & $40.8\pm0.8$ \\
 & MoRe\_L2 & 64 & $27.7\pm2.3$ & $33.4\pm2.2$ & $40.1\pm1.1$ & $43.9\pm0.7$ & $49.3\pm0.7$ \\
 & MoRe\_L3 & 64 & $43.6\pm2.3$ & $48.7\pm2.3$ & $56.0\pm1.3$ & $62.2\pm1.2$ & $65.4\pm0.8$ \\
\midrule
\multirow{6}{*}{Wikipedia} & PCA & 256 & $39.1\pm5.9$ & $44.7\pm2.4$ & $49.5\pm2.5$ & $54.6\pm1.3$ & $57.9\pm1.2$ \\
 & MoRe\_all & 256 & $48.4\pm8.5$ & $55.0\pm1.8$ & $57.4\pm1.7$ & $61.8\pm1.9$ & $61.9\pm1.4$ \\
 & MoRe\_L0 & 64 & $42.4\pm5.4$ & $48.9\pm4.0$ & $51.5\pm3.1$ & $54.6\pm1.3$ & $58.6\pm1.3$ \\
 & MoRe\_L1 & 64 & $48.2\pm8.3$ & $54.4\pm3.4$ & $56.0\pm1.3$ & $58.6\pm2.4$ & $62.6\pm1.2$ \\
 & MoRe\_L2 & 64 & $39.5\pm9.2$ & $45.4\pm3.0$ & $48.4\pm1.6$ & $51.0\pm1.5$ & $54.6\pm1.0$ \\
 & MoRe\_L3 & 64 & $37.3\pm5.2$ & $43.1\pm3.9$ & $48.5\pm3.7$ & $55.4\pm2.6$ & $59.5\pm1.3$ \\
\bottomrule
\end{tabular}
\end{table}

\subsubsection{Category Concentration}
\label{app:concentration}

For every coordinate, concentration is the coefficient of variation of its activation across the $K$ label categories, normalized by $\sqrt{K-1}$.

\begin{table}[H]
\centering
\small
\setlength{\tabcolsep}{5pt}
\caption{Mean category concentration per layer (observed / shuffled-label null; 64 coordinates per layer).}
\label{tab:concentration_values}
\begin{tabular}{lcccc}
\toprule
Dataset & L0 & L1 & L2 & L3 \\
\midrule
AMPS      & 0.0542 / 0.0021 & 0.0458 / 0.0017 & 0.0406 / 0.0016 & 0.0392 / 0.0014 \\
arXiv     & 0.0289 / 0.0078 & 0.0160 / 0.0055 & 0.0244 / 0.0068 & 0.0501 / 0.0119 \\
Wikipedia & 0.0258 / 0.0052 & 0.0287 / 0.0050 & 0.0285 / 0.0076 & 0.0364 / 0.0129 \\
\bottomrule
\end{tabular}
\end{table}

\subsubsection{Extracted Low-to-High Layer Relations}
\label{app: cases}

Table~\ref{tab:blinks} lists representative cross-layer links on AMPS, obtained from the estimated transition weights between coordinates of different layers.

\begin{table}[H]
\centering
\caption{%
  Representative cross-layer transition links on AMPS
  (order $p=10$, $\lambda_{\mathrm{sp}}=10^{-4}$, Pythia-1.4B).
  Each row shows a directed temporal dependency from a source-layer
  concept to a target-layer concept, annotated with the top-3
  activating categories and the mean transition weight (strength).
  The links follow the curriculum direction (arithmetic
  $\to$ advanced linear algebra).
}
\label{tab:blinks}
\setlength{\tabcolsep}{5pt}
\begin{tabular}{llp{3.8cm}p{3.8cm}r}
\toprule
\textbf{Dataset} &
\textbf{Block} &
\textbf{Source concept (top categories)} &
\textbf{Target concept (top categories)} &
\textbf{Strength} \\
\midrule
\multicolumn{5}{l}{\textit{AMPS — curriculum hierarchy (arithmetic $\to$ advanced linear algebra)}} \\
\addlinespace[2pt]
AMPS & L0\,$\to$\,L2
  & \texttt{L0\_add}, \texttt{L0\_simplify}, \texttt{L1\_factor}
  & \texttt{L3\_determinant}, \texttt{L2\_matrix}, \texttt{L1\_equation}
  & 0.078 \\
AMPS & L0\,$\to$\,L1
  & \texttt{L1\_equation}, \texttt{L1\_factor}, \texttt{L0\_simplify}
  & \texttt{L0\_multiply}, \texttt{L1\_equation}, \texttt{L2\_matrix}
  & 0.063 \\
AMPS & L1\,$\to$\,L3
  & \texttt{L1\_factor}, \texttt{L0\_simplify}, \texttt{L0\_add}
  & \texttt{L3\_determinant}, \texttt{L2\_matrix}, \texttt{L0\_multiply}
  & 0.045 \\
\bottomrule
\end{tabular}
\end{table}

\section{Extended Related Work}
\label{app:extended-related-work}
We review the literature from a relatively high-level perspective to show that a principled approach to continual representation learning is missing and that the key may be a modular representation structure.
\paragraph{Supervised Continual Learning}
Continual learning (CL) is commonly framed as the problem of acquiring new capabilities without catastrophic forgetting. In supervised CL, many representative methods address this problem by constraining parameter updates, replaying information from previous tasks, or isolating parameters across tasks. Regularization and distillation methods---such as Elastic Weight Consolidation \citep{kirkpatrick2017overcoming}, Learning without Forgetting \citep{li2016learning} and Memory Aware Synapses \citep{aljundi2018memory}---penalize changes that would harm previous tasks. Replay-based and exemplar-based methods, such as iCaRL \citep{rebuffi2017icarl}, maintain or approximate past-task information while jointly learning classifiers and representations. Architecture-based methods instead allocate separate or partially separate computation to different tasks: Progressive Networks add task-specific columns with lateral transfer, PathNet learns task-specific paths through a shared module pool, and Expert Gate learns a gating mechanism that routes inputs to task-specialized experts \citep{rusu2016progressive,fernando2017pathnet,aljundi2017expert}. A closely related line of work explicitly conditions computation on the task, such as Hard Attention to the Task \citep{serra2018overcoming}, which learns nearly binary attention masks that protect units used by previous tasks.

The Mixture-of-Experts (MoE) framework \citep{shazeer2017outrageously}, in which a trainable gating network selects a sparse combination of feed-forward experts per example, provides a natural substrate for continual learning since new experts can be added without disturbing old ones. Recent analyses further formalize how expert specialization and routing reduce interference in continual task streams \citep{li2025theory}. Lifelong-MoE \citep{chen2023lifelong} expands and freezes distribution-specialized experts and gating dimensions for continual language pre-training, while MoE-Adapters \citep{yu2024boosting} attach task-specific adapter experts to a frozen vision-language backbone with a distribution-discriminative auto-selector. With large pre-trained models, parameter-efficient fine-tuning (PEFT) methods adapt only small task-specific components, such as adapters or low-rank updates, while freezing most backbone parameters \citep{houlsby2019parameter,hu2022lora}. Building on this, recent CL methods constrain successive PEFT updates to interfere as little as possible with prior knowledge: O-LoRA \citep{wang2023orthogonal} learns sequential LoRA updates in mutually orthogonal low-rank subspaces, and InfLoRA \citep{liang2024inflora} reparameterizes pre-trained weights through an interference-eliminating subspace. Prompt-based continual learning extends the PEFT idea by learning small prompt memories or complementary prompts to manage task-specific and task-invariant knowledge without replay \citep{wang2022learning,wang2022dualprompt,smith2023coda}.
Despite their empirical strengths, these supervised CL methods predominantly treat forgetting as a parameter- or task-level phenomenon: experts, paths, prompts, and low-rank updates are defined and indexed by \emph{tasks} rather than by the latent structure of the data, and consequently offer no principled way to separate fundamental, broadly shared knowledge from task-specific specializations. We argue that the underlying issue is fundamentally representational: tasks require representations that are simultaneously distinct (so updates do not interfere) and structured (so that what is shared is reused). Once such rich representations are obtained, a few labeled samples should suffice to specialize to any new task. MoRe takes this view explicitly, deriving its modular decomposition from identifiable factors of the data-generating process rather than from task labels.

\paragraph{Unsupervised Continual Learning}
A second line of work studies continual learning when labels or task identities are unavailable. CURL \citep{rao2019continual} introduced an unsupervised continual representation-learning setting in which the model infers tasks, dynamically expands to capture new concepts, and uses rehearsal to reduce forgetting. In the SSL setting, \citet{hu2022streaming} evaluate sequential self-supervised pre-training on streaming ImageNet and DomainNet, using downstream transfer tasks to measure whether learned representations both adapt and avoid forgetting; they find that simple replay or parameter regularization can make sequential SSL close to joint pre-training in many settings, while severe distribution shifts remain challenging. LUMP \citep{madaan2021representational} further observes that unsupervised representations forget less than supervised ones and proposes mixup interpolation between current and replayed instances. CaSSLe \citep{fini2022self} converts SSL objectives into temporal distillation by mapping current representations to past ones, improving continual self-supervised visual representation learning across several SSL losses. PFR \citep{gomez2022continually} and Kaizen \citep{tang2024kaizen} extend this predictor-based framework to mitigate plasticity loss and to jointly train feature extractors and classifiers. Most closely related to our motivation, Branch-Tuning \citep{liu2025branch} analyzes stability and plasticity using centered kernel alignment and proposes branch expansion/compression, fixing existing components while adding trainable convolutional branches for new data and then reparameterizing them back into the network. These methods clarify that continual SSL should be evaluated at the representation level, rather than only by immediate training losses. However, their adaptation mechanisms still operate by replay, distillation, regularization, or branch modification. MoRe instead asks what representation structure makes such selective adaptation principled: we posit and identify a hierarchical modular representation, enabling adaptation to be localized to the modules whose latent factors have changed.

\paragraph{Modular Networks for Continual Learning}
Modularity is a long-standing strategy for managing the stability--plasticity trade-off \citep{wang2024comprehensive}. Recent work has unified adapters, MoE layers, sparse subnetworks, and prompt mixtures under a common modular framework that highlights CL as a canonical setting where modularity should help \citep{pfeiffer2023modular}. Progressive networks add a new column for each task and connect it laterally to previous columns \citep{rusu2016progressive}; PathNet evolves task-specific paths through a fixed pool of modules \citep{fernando2017pathnet}; Expert Gate learns a network of experts and uses gating autoencoders to select the relevant expert at test time \citep{aljundi2017expert}; hard-attention and pruning approaches learn task-specific masks or reuse freed parameters to protect old tasks \citep{serra2018overcoming,mallya2018packnet}; and MNTDP \citep{veniat2021efficientcontinuallearningmodular} maintains a library of modules and decides per task which to reuse and which to add via a data-driven prior. Modular adapter approaches extend the same idea to parameter-efficient adaptation of pretrained transformers \citep{ponti2023combining}, and sparse mixture-of-experts architectures similarly route inputs to subsets of capacity \citep{shazeer2017outrageously}. These approaches instantiate \emph{architectural} modularity: modules are usually experts, subnetworks, masks, prompts, or adapters that are introduced or selected according to tasks, heuristics, or learned gates. MoRe is complementary but different: it learns \emph{representational} modularity, where modules correspond to latent components organized by a hierarchy in the data-generating process, with identifiability guarantees obtained from time-delayed cross-layer dependencies. The decomposition is therefore not merely a convenient allocation of parameters but a principled reflection of structure in the data, providing a basis for preserving fundamental knowledge while updating more specific components without requiring a task oracle.

%% file: main.bib
@inproceedings{
    li2025idol,
    title={On the Identification of Temporal Causal Representation with Instantaneous Dependence},
    author={Zijian Li and Yifan Shen and Kaitao Zheng and Ruichu Cai and Xiangchen Song and Mingming Gong and Guangyi Chen and Kun Zhang},
    booktitle={The Thirteenth International Conference on Learning Representations},
    year={2025},
    url={https://openreview.net/forum?id=2efNHgYRvM}
}

@article{yao2022tdrl,
  title={Temporally disentangled representation learning},
  author={Yao, Weiran and Chen, Guangyi and Zhang, Kun},
  journal={Advances in Neural Information Processing Systems},
  volume={35},
  pages={26492--26503},
  year={2022}
}

@article{zhang2024generalsetting,
  title={Causal representation learning from multiple distributions: A general setting},
  author={Zhang, Kun and Xie, Shaoan and Ng, Ignavier and Zheng, Yujia},
  journal={arXiv preprint arXiv:2402.05052},
  year={2024}
}

@article{kirkpatrick2017overcoming,
  title={Overcoming catastrophic forgetting in neural networks},
  author={Kirkpatrick, James and Pascanu, Razvan and Rabinowitz, Neil and Veness, Joel and Desjardins, Guillaume and Rusu, Andrei A and Milan, Kieran and Quan, John and Ramalho, Tiago and Grabska-Barwinska, Agnieszka and others},
  journal={Proceedings of the national academy of sciences},
  volume={114},
  number={13},
  pages={3521--3526},
  year={2017},
  publisher={National Academy of Sciences}
}

@inproceedings{aljundi2018memory,
  title={Memory aware synapses: Learning what (not) to forget},
  author={Aljundi, Rahaf and Babiloni, Francesca and Elhoseiny, Mohamed and Rohrbach, Marcus and Tuytelaars, Tinne},
  booktitle={Proceedings of the European conference on computer vision (ECCV)},
  pages={139--154},
  year={2018}
}

@inproceedings{rebuffi2017icarl,
  title={icarl: Incremental classifier and representation learning},
  author={Rebuffi, Sylvestre-Alvise and Kolesnikov, Alexander and Sperl, Georg and Lampert, Christoph H},
  booktitle={Proceedings of the IEEE conference on Computer Vision and Pattern Recognition},
  pages={2001--2010},
  year={2017}
}

@article{rusu2016progressive,
  title   = {Progressive Neural Networks},
  author  = {Rusu, Andrei A. and Rabinowitz, Neil C. and Desjardins, Guillaume and Soyer, Hubert and Kirkpatrick, James and Kavukcuoglu, Koray and Pascanu, Razvan and Hadsell, Raia},
  journal = {arXiv preprint arXiv:1606.04671},
  year    = {2016}
}

@inproceedings{aljundi2017expert,
  title={Expert gate: Lifelong learning with a network of experts},
  author={Aljundi, Rahaf and Chakravarty, Punarjay and Tuytelaars, Tinne},
  booktitle={Proceedings of the IEEE conference on computer vision and pattern recognition},
  pages={3366--3375},
  year={2017}
}

@inproceedings{serra2018overcoming,
  title={Overcoming catastrophic forgetting with hard attention to the task},
  author={Serra, Joan and Suris, Didac and Miron, Marius and Karatzoglou, Alexandros},
  booktitle={International conference on machine learning},
  pages={4548--4557},
  year={2018},
  organization={PMLR}
}

@inproceedings{mallya2018packnet,
  title={Packnet: Adding multiple tasks to a single network by iterative pruning},
  author={Mallya, Arun and Lazebnik, Svetlana},
  booktitle={Proceedings of the IEEE conference on Computer Vision and Pattern Recognition},
  pages={7765--7773},
  year={2018}
}

@article{shazeer2017outrageously,
  title={Outrageously large neural networks: The sparsely-gated mixture-of-experts layer},
  author={Shazeer, Noam and Mirhoseini, Azalia and Maziarz, Krzysztof and Davis, Andy and Le, Quoc and Hinton, Geoffrey and Dean, Jeff},
  journal={arXiv preprint arXiv:1701.06538},
  year={2017}
}

@inproceedings{chen2023lifelong,
  title={Lifelong language pretraining with distribution-specialized experts},
  author={Chen, Wuyang and Zhou, Yanqi and Du, Nan and Huang, Yanping and Laudon, James and Chen, Zhifeng and Cui, Claire},
  booktitle={International Conference on Machine Learning},
  pages={5383--5395},
  year={2023},
  organization={PMLR}
}

@inproceedings{yu2024boosting,
  title={Boosting continual learning of vision-language models via mixture-of-experts adapters},
  author={Yu, Jiazuo and Zhuge, Yunzhi and Zhang, Lu and Hu, Ping and Wang, Dong and Lu, Huchuan and He, You},
  booktitle={Proceedings of the IEEE/CVF Conference on Computer Vision and Pattern Recognition},
  pages={23219--23230},
  year={2024}
}

@inproceedings{houlsby2019parameter,
  title={Parameter-efficient transfer learning for NLP},
  author={Houlsby, Neil and Giurgiu, Andrei and Jastrzebski, Stanislaw and Morrone, Bruna and De Laroussilhe, Quentin and Gesmundo, Andrea and Attariyan, Mona and Gelly, Sylvain},
  booktitle={International conference on machine learning},
  pages={2790--2799},
  year={2019},
  organization={PMLR}
}

@article{hu2022lora,
  title={Lora: Low-rank adaptation of large language models.},
  author={Hu, Edward J and Shen, Yelong and Wallis, Phillip and Allen-Zhu, Zeyuan and Li, Yuanzhi and Wang, Shean and Wang, Liang and Chen, Weizhu and others},
  journal={Iclr},
  volume={1},
  number={2},
  pages={3},
  year={2022}
}

@inproceedings{wang2023orthogonal,
  title={Orthogonal subspace learning for language model continual learning},
  author={Wang, Xiao and Chen, Tianze and Ge, Qiming and Xia, Han and Bao, Rong and Zheng, Rui and Zhang, Qi and Gui, Tao and Huang, Xuan-Jing},
  booktitle={Findings of the Association for Computational Linguistics: EMNLP 2023},
  pages={10658--10671},
  year={2023}
}

@inproceedings{liang2024inflora,
  title={Inflora: Interference-free low-rank adaptation for continual learning},
  author={Liang, Yan-Shuo and Li, Wu-Jun},
  booktitle={Proceedings of the IEEE/CVF Conference on Computer Vision and Pattern Recognition},
  pages={23638--23647},
  year={2024}
}

@inproceedings{wang2022learning,
  title={Learning to prompt for continual learning},
  author={Wang, Zifeng and Zhang, Zizhao and Lee, Chen-Yu and Zhang, Han and Sun, Ruoxi and Ren, Xiaoqi and Su, Guolong and Perot, Vincent and Dy, Jennifer and Pfister, Tomas},
  booktitle={Proceedings of the IEEE/CVF conference on computer vision and pattern recognition},
  pages={139--149},
  year={2022}
}

@inproceedings{wang2022dualprompt,
  title={Dualprompt: Complementary prompting for rehearsal-free continual learning},
  author={Wang, Zifeng and Zhang, Zizhao and Ebrahimi, Sayna and Sun, Ruoxi and Zhang, Han and Lee, Chen-Yu and Ren, Xiaoqi and Su, Guolong and Perot, Vincent and Dy, Jennifer and others},
  booktitle={European conference on computer vision},
  pages={631--648},
  year={2022},
  organization={Springer}
}

@inproceedings{smith2023coda,
  title={Coda-prompt: Continual decomposed attention-based prompting for rehearsal-free continual learning},
  author={Smith, James Seale and Karlinsky, Leonid and Gutta, Vyshnavi and Cascante-Bonilla, Paola and Kim, Donghyun and Arbelle, Assaf and Panda, Rameswar and Feris, Rogerio and Kira, Zsolt},
  booktitle={Proceedings of the IEEE/CVF conference on computer vision and pattern recognition},
  pages={11909--11919},
  year={2023}
}

@article{wang2024comprehensive,
  title={A comprehensive survey of continual learning: Theory, method and application},
  author={Wang, Liyuan and Zhang, Xingxing and Su, Hang and Zhu, Jun},
  journal={IEEE transactions on pattern analysis and machine intelligence},
  volume={46},
  number={8},
  pages={5362--5383},
  year={2024},
  publisher={IEEE}
}

@article{rao2019continual,
  title={Continual unsupervised representation learning},
  author={Rao, Dushyant and Visin, Francesco and Rusu, Andrei and Pascanu, Razvan and Teh, Yee Whye and Hadsell, Raia},
  journal={Advances in neural information processing systems},
  volume={32},
  year={2019}
}

@article{hu2021well,
  title={How well does self-supervised pre-training perform with streaming data?},
  author={Hu, Dapeng and Yan, Shipeng and Lu, Qizhengqiu and Hong, Lanqing and Hu, Hailin and Zhang, Yifan and Li, Zhenguo and Wang, Xinchao and Feng, Jiashi},
  journal={arXiv preprint arXiv:2104.12081},
  year={2021}
}

@article{madaan2021representational,
  title={Representational continuity for unsupervised continual learning},
  author={Madaan, Divyam and Yoon, Jaehong and Li, Yuanchun and Liu, Yunxin and Hwang, Sung Ju},
  journal={arXiv preprint arXiv:2110.06976},
  year={2021}
}

@inproceedings{gomez2022continually,
  title={Continually learning self-supervised representations with projected functional regularization},
  author={Gomez-Villa, Alex and Twardowski, Bartlomiej and Yu, Lu and Bagdanov, Andrew D and Van de Weijer, Joost},
  booktitle={Proceedings of the IEEE/CVF conference on computer vision and pattern recognition},
  pages={3867--3877},
  year={2022}
}

@inproceedings{tang2024kaizen,
  title={Kaizen: Practical self-supervised continual learning with continual fine-tuning},
  author={Tang, Chi Ian and Qendro, Lorena and Spathis, Dimitris and Kawsar, Fahim and Mascolo, Cecilia and Mathur, Akhil},
  booktitle={Proceedings of the IEEE/CVF winter conference on applications of computer vision},
  pages={2841--2850},
  year={2024}
}

@inproceedings{fini2022self,
  title={Self-supervised models are continual learners},
  author={Fini, Enrico and Da Costa, Victor G Turrisi and Alameda-Pineda, Xavier and Ricci, Elisa and Alahari, Karteek and Mairal, Julien},
  booktitle={Proceedings of the IEEE/CVF conference on computer vision and pattern recognition},
  pages={9621--9630},
  year={2022}
}

@article{liu2025branch,
  title={Branch-tuning: balancing stability and plasticity for continual self-supervised learning},
  author={Liu, Wenzhuo and Zhu, Fei and Liu, Cheng-Lin},
  journal={IEEE Transactions on Neural Networks and Learning Systems},
  year={2025},
  publisher={IEEE}
}

@article{hu2022streaming,
  title={How well does self-supervised pre-training perform with streaming data?},
  author={Hu, Dapeng and Yan, Shipeng and Lu, Qizhengqiu and Hong, Lanqing and Hu, Hailin and Zhang, Yifan and Li, Zhenguo and Wang, Xinchao and Feng, Jiashi},
  journal={arXiv preprint arXiv:2104.12081},
  year={2021}
}

@InProceedings{li2016learning,
author="Li, Zhizhong
and Hoiem, Derek",
editor="Leibe, Bastian
and Matas, Jiri
and Sebe, Nicu
and Welling, Max",
title="Learning Without Forgetting",
booktitle="Computer Vision -- ECCV 2016",
year="2016",
publisher="Springer International Publishing",
address="Cham",
pages="614--629",
}

@article{pfeiffer2023modular,
  title={Modular deep learning},
  author={Pfeiffer, Jonas and Ruder, Sebastian and Vuli{\'c}, Ivan and Ponti, Edoardo Maria},
  journal={arXiv preprint arXiv:2302.11529},
  year={2023}
}

@misc{veniat2021efficientcontinuallearningmodular,
      title={Efficient Continual Learning with Modular Networks and Task-Driven Priors}, 
      author={Tom Veniat and Ludovic Denoyer and Marc'Aurelio Ranzato},
      year={2021},
      eprint={2012.12631},
      archivePrefix={arXiv},
      primaryClass={cs.LG},
}

@inproceedings{ponti2023combining,
  title={Combining parameter-efficient modules for task-level generalisation},
  author={Ponti, Edoardo Maria and Sordoni, Alessandro and Bengio, Yoshua and Reddy, Siva},
  booktitle={Proceedings of the 17th Conference of the European Chapter of the Association for Computational Linguistics},
  pages={687--702},
  year={2023}
}

@article{fernando2017pathnet,
  title={Pathnet: Evolution channels gradient descent in super neural networks},
  author={Fernando, Chrisantha and Banarse, Dylan and Blundell, Charles and Zwols, Yori and Ha, David and Rusu, Andrei A and Pritzel, Alexander and Wierstra, Daan},
  journal={arXiv preprint arXiv:1701.08734},
  year={2017}
}

@article{li2025theory,
  title={Theory on mixture-of-experts in continual learning},
  author={Li, Hongbo and Lin, Sen and Duan, Lingjie and Liang, Yingbin and Shroff, Ness B},
  journal={arXiv preprint arXiv:2406.16437},
  year={2024}
}

@inproceedings{biderman2023pythia,
  title={Pythia: A suite for analyzing large language models across training and scaling},
  author={Biderman, Stella and Schoelkopf, Hailey and Anthony, Quentin Gregory and Bradley, Herbie and O’Brien, Kyle and Hallahan, Eric and Khan, Mohammad Aflah and Purohit, Shivanshu and Prashanth, USVSN Sai and Raff, Edward and others},
  booktitle={International conference on machine learning},
  pages={2397--2430},
  year={2023},
  organization={PMLR}
}

@article{tafazoli2026building,
  title={Building compositional tasks with shared neural subspaces},
  author={Tafazoli, Sina and Bouchacourt, Flora M and Ardalan, Adel and Markov, Nikola T and Uchimura, Motoaki and Mattar, Marcelo G and Daw, Nathaniel D and Buschman, Timothy J},
  journal={Nature},
  volume={650},
  number={8100},
  pages={164--172},
  year={2026},
  publisher={Nature Publishing Group}
}

@article{tian2026domain,
  title={Domain-specific schema reuse supports flexible learning to learn in the primate brain},
  author={Tian, Kaixi and Zhao, Zhiping and Chen, Yang and Ge, Ningling and Cao, Shenghao and Han, Xinyong and Gu, Jianwen and Yu, Shan},
  journal={Nature Communications},
  year={2026},
  publisher={Nature Publishing Group}
}
